\documentclass[svgnames]{article}

\usepackage{microtype}
\usepackage{graphicx}
\usepackage{booktabs} 
\usepackage{multirow}
\usepackage{adjustbox}
\usepackage{paralist}
\usepackage{amsmath}
\usepackage{amsthm}
\usepackage{amssymb}
\usepackage{amsmath}
\usepackage{bm}
\usepackage{mathrsfs}
\usepackage{dsfont} 
\usepackage{centernot}
\usepackage{nicefrac}       
\usepackage{mathtools} 
\usepackage{paralist}
\usepackage{enumitem}
\usepackage{thmtools, thm-restate}
\usepackage{caption}
\usepackage{subcaption}

\declaretheorem{lemma}

\declaretheorem{definition}
\usepackage{thm-restate}
\usepackage{tikz-cd}
\usepackage{lineno}
\usepackage{algorithm}
\usepackage[noend]{algpseudocode}
\usepackage{xcolor}
\usepackage{colortbl}
\usepackage{tabularx}
\usepackage{pdfpages}
\usepackage{tikz}
\usepackage{wrapfig}

\DeclareMathOperator{\Tr}{Tr}
\DeclareMathOperator{\softmaxOp}{softmax}
\newcommand{\softmax}[1]{\softmaxOp\left(#1\right)}

\def\grp{{G}}

\def\op{\cdot}
\def\id{e}

\usepackage{hyperref}

\usepackage[accepted]{icml2021}

\icmltitlerunning{Equivariant Self-Attention for Lie Groups}

\usepackage{selectp}
\begin{document}

\twocolumn[
\icmltitle{LieTransformer: Equivariant Self-Attention for Lie Groups}



\icmlsetsymbol{equal}{*}

\begin{icmlauthorlist}
\icmlauthor{Michael Hutchinson}{equal,ox}
\icmlauthor{Charline Le Lan}{equal,ox}
\icmlauthor{Sheheryar Zaidi}{equal,ox} \\
\icmlauthor{Emilien Dupont}{ox} 
\icmlauthor{Yee Whye Teh}{ox,dm}
\icmlauthor{Hyunjik Kim}{dm}
\end{icmlauthorlist}

\icmlaffiliation{ox}{Department of Statistics, University of Oxford, UK}
\icmlaffiliation{dm}{DeepMind, UK}

\icmlcorrespondingauthor{Michael Hutchinson}{hutchinson.michael.john@gmail.com}
\icmlcorrespondingauthor{Charline Le Lan}{charline.lelan@stats.ox.ac.uk}
\icmlcorrespondingauthor{Sheheryar Zaidi}{sheh\_zaidi96@hotmail.com}

\icmlkeywords{symmetries, equivariance, self-attention, transformer}

\vskip 0.3in
]



\printAffiliationsAndNotice{\icmlEqualContribution} 

\begin{abstract}
Group equivariant neural networks are used as building blocks of group invariant neural networks, which have been shown to improve generalisation performance and data efficiency through principled parameter sharing. Such works have mostly focused on group equivariant convolutions, building on the result that group equivariant linear maps are necessarily convolutions. In this work, we extend the scope of the literature to \textit{self-attention}, that is emerging as a prominent building block of deep learning models. We propose the \verb!LieTransformer!, an architecture composed of \verb!LieSelfAttention! layers that are equivariant to arbitrary Lie groups and their discrete subgroups. We demonstrate the generality of our approach by showing experimental results that are competitive to baseline methods on a wide range of tasks: shape counting on point clouds, molecular property regression and modelling particle trajectories under Hamiltonian dynamics.
\end{abstract}

\section{Introduction}
Group equivariant neural networks are useful architectures for problems with symmetries that can be described in terms of a group (in the mathematical sense). Convolutional neural networks (CNNs) are a special case that deal with translational symmetry, in that when the input to a convolutional layer is translated, the output is also translated. This property is known as \textit{translation equivariance}, and offers a useful inductive bias for perception tasks which usually have translational symmetry. Constraining a linear layer to obey this symmetry, resulting in a covolutional layer, greatly reduces the number of parameters and computational cost. This has led to the success of CNNs in multiple domains such as computer vision \citep{krizhevsky2012imagenet} and audio \citep{graves2014towards}. Following on from this success, there has been a growing literature on the study of group equivariant CNNs (G-CNNs) that generalise CNNs to deal with other types of symmetries beyond translations, such as rotations and reflections. 

Most works on group equivariant NNs deal with CNNs i.e.~linear maps with shared weights composed with pointwise non-linearities, building on the result that group equivariant linear maps (with mild assumptions) are necessarily convolutions \citep{kondor2018generalization, cohen2019general, bekkers2020b}. However there has been little work on non-linear group equivariant building blocks. In this paper we extend group equivariance to self-attention \citep{vaswani2017attention}, a non-trivial non-linear map, that has become a prominent building block of deep learning models in various data modalities, such as natural-language processing \citep{vaswani2017attention,  brown2020language}, computer vision \citep{zhang2018self, ramachandran2019stand}, reinforcement learning \citep{parisotto2019stabilizing}, and audio generation \citep{huang2018music}.

We thus propose \verb!LieTransformer!, a group invariant Transformer built from group equivariant \verb!LieSelfAttention! layers. It uses a lifting based approach, that relaxes constraints on the attention module compared to approaches without lifting. Our method is applicable to Lie groups and their discrete subgroups (e.g.~cyclic groups $C_n$ and dihedral groups $D_n$) acting on homogeneous spaces.
Our work is very much in the spirit of \citet{finzi2020generalizing}, our main baseline, but for group equivariant self-attention instead of convolutions.
Among works that deal with equivariant self-attention, we are the first to propose a methodology for general groups and domains (unspecified to 2D images \cite{romero2020attentive, romero2020group} or 3D point clouds \cite{fuchs2020se}).
We demonstrate the generality of our approach through strong performance on a wide variety of tasks, namely shape counting on point clouds, molecular property regression and modelling particle trajectories under Hamiltonian dynamics.

\section{Background}
\label{sec:background}
\subsection{Group Equivariance}
This section lays down some of the necessary definitions and notations in group theory and representation theory in an informal and intuitive manner. For a more formal presentation of definitions, see \autoref{sec:def}.

Loosely speaking, a \textbf{group} $G$ is a set of symmetries, with each group element $g$ corresponding to a symmetry transformation. These group elements ($g,g' \in G$) can be composed ($gg'$) or inverted ($g^{-1}$), just like transformations. An example of a \textbf{discrete group} is $C_n$, the set of rotational symmetries of a regular $n$-gon. The group consists of $n$ such rotations, including the identity. An example of a continuous (infinite) group is $SO(2)$, the set of all 2D rotations about the origin. $C_n$ is a subset of $SO(2)$, hence we call $C_n$ a \textbf{subgroup} of $SO(2)$. Note that $SO(2) = \{g_\theta : \theta \in [0,2\pi) \}$ can be parameterised by the angle of rotation $\theta$. Such groups that can be continuously parameterised by real values are called \textbf{Lie groups}.

A symmetry transformation of group element $g \in G$ on object $v \in V$ is referred to as the \textbf{group action} of $G$ on $V$. If this action is linear on a vector space $V$, then we can represent the action as a linear map $\rho(g)$. We call $\rho$ a \textbf{representation} of $G$, and $\rho(g)$ often takes the form of a matrix.
For $SO(2)$, the standard rotation matrix is an example of a representation that acts on $V=\mathbb{R}^2$:
\begin{equation}
    \rho(g_\theta) = \begin{bmatrix}
    \cos \theta & -\sin \theta \\
    \sin \theta  & \cos \theta 
    \end{bmatrix}
\end{equation}
Note that this is only one of many possible representations of $SO(2)$ acting on $\mathbb{R}^2$ (e.g.~replacing $\theta$ with $n \theta$ yields another valid representation), and $SO(2)$ can act on spaces other than $\mathbb{R}^2$, e.g.~$\mathbb{R}^d$ for arbitrary $d \geq 2$.

In the context of group equivariant neural networks, $V$ is commonly defined to be the space of scalar-valued functions on some set $S$, so that $V = \{f \mid f:S \to \mathbb{R}\}$. This set could be a Euclidean input space e.g.~a grey-scale image can be expressed as a feature map $f: \mathbb{R}^2 \rightarrow \mathbb{R}$ from pixel coordinate $x_i$ to pixel intensity $\mathtt{f}_i$, supported on the grid of pixel coordinates.  We may express the rotation of the image as a representation of $SO(2)$ by extending the action $\rho$ on the pixel coordinates to a representation $\pi$ that acts on the space of feature maps: 
\begin{equation} \label{eq:regular_x}
    [\pi(g_\theta)(f)](x) \triangleq f (\rho(g_\theta^{-1})x).
\end{equation}
Note that this is equivalent to the mapping $(x_i, \mathtt{f}_i)_{i=1}^n \mapsto (\rho(g_{\theta})x_i, \mathtt{f}_i)_{i=1}^n$.
As a special case, we can define $V = \{f| f: \grp \rightarrow \mathbb{R} \}$ to be the space of scalar-valued functions on the group $\grp$, for which we can define a representation $\pi$ acting on $V$ via the \textbf{regular representation}:
\begin{equation} \label{eq:regular}
    [\pi(g_\theta)(f)](g_{\phi}) \triangleq f (g_\theta^{-1} g_{\phi}).
\end{equation}
Here the action $\rho$ is replaced by the action of the group on itself.
If we wish to handle multiple channels of data, e.g. RGB images, we can stack  these feature maps together, transforming in a similar manner.

Now let us define the notion of \textbf{$G$-equivariance}.
\begin{definition}
We say that a map $\Phi: V_1 \rightarrow V_2$ is \textbf{$G$-equivariant} with respect to actions $\rho_1, \rho_2$ of $G$ acting on $V_1,V_2$ respectively if: $\Phi [\rho_1(g) f] = \rho_2(g) \Phi[f]$ for any $g \in G, f \in V_1$. 
\end{definition}

In the above example of rotating RGB images, we have $G=SO(2)$ and $\rho_1=\rho_2=\pi$. Hence the equivariance of $\Phi$ with respect to $SO(2)$ means that rotating an input image and then applying $\Phi$ yields the same result as first applying $\Phi$ to the original input image and then rotating the output, i.e.~$\Phi$ \textit{commutes} with the representation $\pi$.

The end goal for group equivariant neural networks is to design a neural network that obeys certain symmetries in the data. For example, we may want an image classifier to output the same classification when the input image is rotated. So in fact we want a \textbf{$G$-invariant} neural network, where the output is invariant to group actions on the input space. Note that $G$-invariance is a special case of $G$-equivariance, where $\rho_2$ is the \textbf{trivial representation} i.e.~$\rho_2(g)$ is the identity map for any $g \in G$. Invariant maps are easy to design, by discarding information, e.g.~pooling over spatial dimensions is invariant to rotations and translations. However, such maps are not expressive as they fail to extract high-level semantic features from the data. This is where equivariant neural networks become relevant; the standard recipe for constructing an expressive invariant neural network is to compose multiple equivariant layers with a final invariant layer. It is a standard result that such maps are invariant (e.g.~\citet{bloem2020probabilistic})
and a proof is given in \autoref{sec:proof} for completeness.

\begin{figure*}[t]
    \centering
    \includegraphics[width=\textwidth]{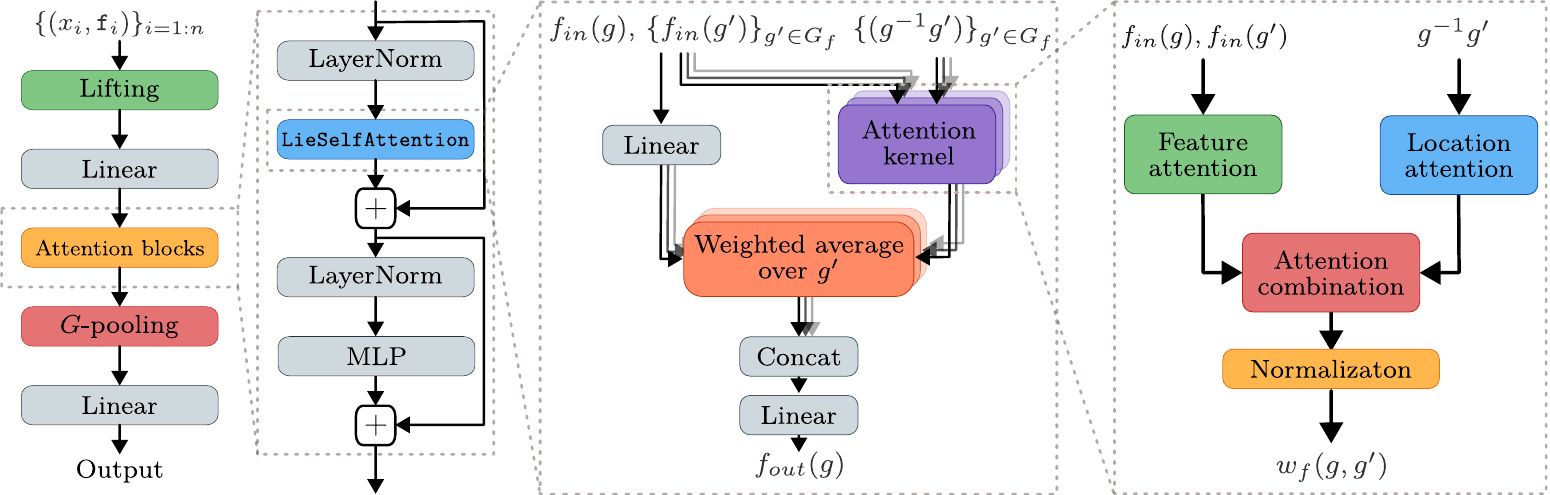}
    \caption{Architecture of the LieTransformer.}
    \label{fig:architecture}
    \vspace{-1em}
\end{figure*}

\subsection{Equivariant Maps on Homogeneous Input Spaces} \label{sec:hom}
Here we introduce the framework for $G$-equivariant maps, and provide group equivariant convolutions as an example. Suppose we have data in the form of a set of input pairs $(x_i,\mathtt{f}_i)_{i=1}^n$ where $x_i \in \mathcal{X}$ are spatial coordinates and $\mathtt{f}_i \in \mathcal{F}$ are feature values. The data can be described as a single feature map $f_{\mathcal{X}}: x_i \mapsto \mathtt{f}_i$. We assume that a group $G$ acts on the $x$-space $\mathcal{X}$ via action $\rho$, and that the action is \textbf{transitive} (also referred to as $\mathcal{X}$ being \textbf{homogeneous}). This means that all elements of $\mathcal{X}$ are connected by the action: $\forall x,x' \in \mathcal{X}$, $\exists g \in G: \rho(g)x = x'$. We often write $gx$ instead of $\rho(g)x$ to reduce clutter. For example, the group of 2D translations $T(2)$ acts transitively on $\mathbb{R}^2$ since there is a translation connecting any two points in $\mathbb{R}^2$. On the other hand, the group of 2D rotations about the origin $SO(2)$ does not act transitively on $\mathbb{R}^2$, since points that have different distances to the origin cannot be mapped onto each other by rotations. However the group of 2D roto-translations $SE(2)$, whose elements can be written as a composition $tR$ of $t \in T(2)$ and $R \in SO(2)$, acts transitively on $\mathbb{R}^2$ since $SE(2)$ contains $T(2)$.

For such homogeneous spaces $\mathcal{X}$, it can be shown that there is a natural partition of $G$ into disjoint subsets such that there is a one-to-one correspondence between $\mathcal{X}$ and these subsets. Namely each $x \in \mathcal{X}$ corresponds to the \textbf{coset} $s(x)H=\{s(x)h | h \in H\}$, where the subgroup $H= \{g \in G|gx_0 = x_0\}$ is called the \textbf{stabiliser} of origin $x_0$, and $s(x) \in G $ is a group element that maps $x_0$ to $x$. It can be shown that the coset $s(x)H$ does not depend on the choice of $s(x)$, and that $s(x)H$ and $s(x')H$ are disjoint for $x \neq x'$. For $T(2)$ acting on $\mathbb{R}^2$, we have $H=\{e\}$, the identity, and $s(x)=t_x$, the group element describing the translation from $x_0$ to $x$, and so each $x$ corresponds to $\{t_x\}$. For $SE(2)$ acting on $\mathbb{R}^2$, we have $H=SO(2)$ and $s(x)=t_x$, so each $x$ corresponds to $\{t_x R| R \in SO(2)\}$. This correspondence is often written as an isomorphism $X \simeq G/H$, where $G/H$ is the set of cosets of $H$.

Using this isomorphism, we can map each point in $\mathcal{X}$ to a set of group elements in $G$, i.e.~mapping each data pair $(x_i, \mathtt{f}_i)$ to (possibly multiple) pairs $\{(g, \mathtt{f}_i) | g \in s(x_i)H\}$. This can be thought of as \textbf{lifting} the feature map $f_{\mathcal{X}}: x_i \mapsto \mathtt{f}_i$ defined on $\mathcal{X}$ to a feature map $\mathcal{L}[f_{\mathcal{X}}]: g \mapsto \mathtt{f}_i$ defined on $G$ \citep{kondor2018generalization}. Let $\mathcal{I}_U$ denote the space of such feature maps from $G$ to $\mathcal{F}$. Subsequently, we may define group equivariant maps as functions from $\mathcal{I}_U$ to itself, which turns out to be a simpler task than defining equivariant maps directly on $\mathcal{X}$.

The \textbf{group equivariant convolution} \citep{cohen2016group,cohen2018spherical,finzi2020generalizing,romero2020attentive} is an example of such a group equivariant map that has been studied extensively. 
Specifically, the group equivariant convolution $\Psi:\mathcal{I}_U \rightarrow \mathcal{I}_U$ is defined as:
\begin{equation} \label{eq:conv}
    [\Psi f](g) \triangleq \int_{G} \psi(g'^{-1}g)f(g') dg'
\vspace{-0.6mm}
\end{equation}%
where $\psi: G \rightarrow \mathbb{R}$ is the convolutional filter and the integral is defined with respect to the left Haar measure of $G$. Note that for discrete groups the integral amounts to a sum over the group. Hence the integral can be computed exactly for discrete groups \citep{cohen2016group}, and for Lie groups it can be approximated using Fast Fourier Transforms \citep{cohen2018spherical} or Monte Carlo (MC) estimation \citep{finzi2020generalizing}. Given the regular representation $\pi$ of $G$ acting on $\mathcal{I}_U$ as in \autoref{eq:regular}, we can easily verify that $\Psi$ is equivariant with respect to $\pi$ (c.f. \autoref{sec:proof}).

\section{LieTransformer}
\label{sec:LieTransformer}
We first outline the problem setting before describing our model, the \verb!LieTransformer!. We tackle the problem of regression/classification for predicting a scalar/vector-valued target $y$ given a set of input pairs $(x_i,\mathtt{f}_i)_{i=1}^n$ where $x_i \in \mathbb{R}^{d_x}$ are spatial locations and $\mathtt{f}_i \in \mathbb{R}^{d_f}$ are feature values at the spatial location. Hence the training data of size $N$ is a set of tuples $((x_i, \mathtt{f}_i)_{i=1}^{n_j}, y_j)_{j=1}^N$. In some tasks such as point cloud classification, the feature values $\mathtt{f}_i$ may not be given. In this case the $\mathtt{f}_i$ can set to be a fixed constant or a ($G$-invariant) function of $(x_i)_{i=1}^n$.

\verb!LieTransformer! is composed of a \textbf{lifting} layer followed by residual blocks of \verb!LieSelfAttention! layers, \verb!LayerNorm! and pointwise \verb!MLP!s, all of which are equivariant with respect to the regular representation, followed by a final invariant \verb!G-pooling! layer (c.f. Appendix \ref{apd:norm} for more details on these layers). We summarise the architecture in \autoref{fig:architecture} and describe its key components below.

\subsection{Lifting} \label{sec:method}

\begin{figure*}[t]
    \centering
    \includegraphics[width=\textwidth]{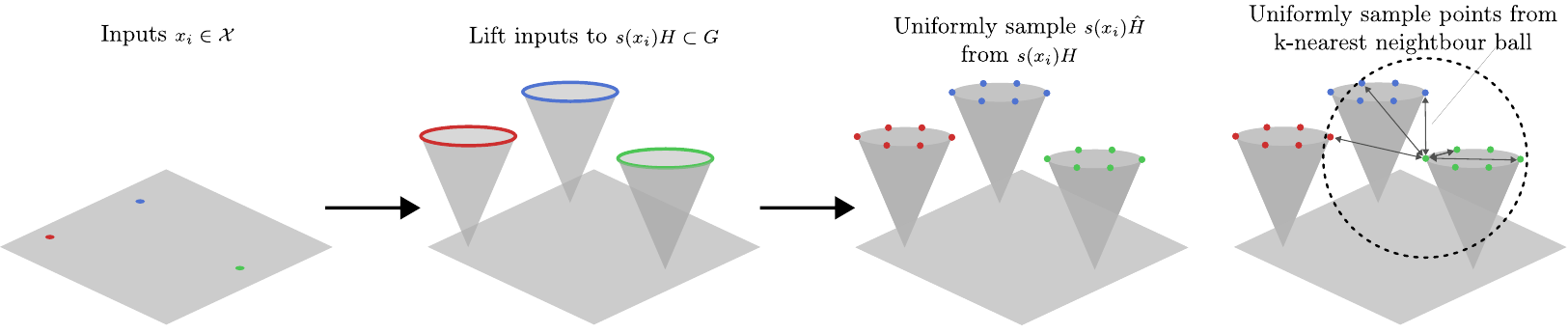}
    \caption{Visualisation of lifting, sampling $\hat{H}$, and subsampling in the local neighbourhood for $SE(2)$ acting on $\mathbb{R}^2$. Self-attention is performed on this subsampled neighbourhood. 
    }
    \label{fig:lifting_sampling}
\end{figure*}

Recall from Section \ref{sec:hom} that the \textbf{lifting} $\mathcal{L}$ maps $f_{\mathcal{X}}$  (supported on $\bigcup_{i=1}^n \{x_i\} \subset \mathcal{X}$) to $\mathcal{L}[f_{\mathcal{X}}]$ (supported on $\bigcup_{i=1}^n s(x_i)H \subset G$) such that:
\begin{equation}
    \mathcal{L}[f_{\mathcal{X}}](g) \triangleq \mathtt{f}_i ~~ \text{for} ~~ g \in s(x_i)H.
\end{equation}
This can be thought of as extending the domain of $f_{\mathcal{X}}$ from $\mathcal{X}$ to $G$ while preserving the feature values $\mathtt{f_i}$, mapping $(x_i, \mathtt{f}_i) \mapsto (g, \mathtt{f}_i)$ for $g \in s(x_i)H$ (c.f. \autoref{fig:lifting_sampling}). Subsequently we may design $G$-equivariant maps on the space of functions on $G$, which is a simpler task than desgining $G$-equivariant maps directly on $\mathcal{X}$ (e.g.~\citet{cohen2018spherical}).

As in Equations \ref{eq:regular_x} and \ref{eq:regular}, we define the representation $\pi$ on $f_{\mathcal{X}}$ and $\mathcal{L}[f_{\mathcal{X}}]$ as:
\begin{gather*}
    [\pi(u)f_{\mathcal{X}}](x) = f_{\mathcal{X}}(u^{-1}x) \\
    [\pi(u)\mathcal{L}[f_{\mathcal{X}}]](g) = \mathcal{L}[f_{\mathcal{X}}](u^{-1}g)
\end{gather*}
where $u \in G$. Note that the actions correspond to mappings $(x_i,\mathtt{f}_i) \mapsto (ux_i,\mathtt{f}_i)$ and $(g, \mathtt{f}_i) \mapsto (ug, \mathtt{f}_i)$ respectively.

We need to ensure that lifting preserves equivariance, which is why we need the space to be homogeneous with respect to the action of $G$ on $\mathcal{X}$.
\begin{restatable}{proposition}{eqlifting}
The lifting layer $\mathcal{L}$ is equivariant with respect to the representation $\pi$.
\end{restatable}
\textit{Intuition for proof}. When $x_i$ is shifted by $u \in G$, the lifted coset $s(x_i)H$ is also shifted by $u$, i.e.~$x_i \mapsto ux_i \Rightarrow s(x_i)H \mapsto us(x_i)H$. See \autoref{sec:proof} for full proof.

\subsection{LieSelfAttention}

\begin{algorithm}
\caption{LieSelfAttention}
\label{alg:lie-self-attention}
\renewcommand{\algorithmicrequire}{\textbf{Input:}}
\renewcommand{\algorithmicensure}{\textbf{Output:}}
\algrenewcommand\algorithmicindent{2em}
\algrenewcommand\algorithmicdo{}
\begin{algorithmic}
\Require $(g, f(g))_{g \in G_f}$
\For{$g \in G_f$}
    \For{$g' \in G_f$ (or $\mathtt{nbhd}_{\eta}(g)$)}
        \State \(\triangleright\) Compute content/location attention
        \State $k_c(f(g),f(g')), k_l(g^{-1}g')$ 
        \State \(\triangleright\) Compute unnormalised weights
        \State $\alpha_f(g,g') = F(k_c(f(g),f(g')), k_l(g^{-1}g'))$
    \EndFor
    \State \(\triangleright\) Compute normalised weights and output
\State $\{w_f(g,g')\}_{g' \in G_f} = \mathtt{norm}{\{\alpha_f(g,g')\}_{g' \in G_f}}$
\State $f_{out}(g) = \int_{G_f} w_f(g,g') W^V f(g') dg'$
\EndFor
\Ensure $(g, f_{out}(g))_{g \in G_f}$
\end{algorithmic}
\end{algorithm}

Let $f \triangleq \mathcal{L}[f_{\mathcal{X}}]$, hence $f$ is defined on the set $G_f = \cup_{i=1}^n s(x_i)H $. We define the \verb!LieSelfAttention! layer in Algorithm \ref{alg:lie-self-attention}, where self-attention (see \autoref{sec:sa} for the original formulation) is defined across the elements of $G_f$. There are various choices for functions content-based attention $k_c$ , location-based attention $k_l$ , $F$ that determines how to combine the two to form unnormalised weights, and the choice of normalisation of weights. See \autoref{sec:sa_choices} for a non-exhaustive list of choices of the above, and also for the details of the multi-head generalisation of \verb!LieSelfAttention!.

\begin{restatable}{proposition}{eqattention}
\label{prop:eqattention}
LieSelfAttention is equivariant with respect to the regular representation $\pi$.
\end{restatable}
\textit{Intuition for proof}. \verb!LieSelfAttention! can be thought of as a map $\Phi: (g, f(g))_{g \in G_f} \mapsto (g, f_{out}(g))_{g \in G_f}$, and equivariance holds if $\forall u \in G$, $\Phi$ maps $(ug, f(g))_{g \in G_f}$ to $(ug, f_{out}(g))_{g \in G_f}$. Now note that $\Phi$ is a function of $g \in G_f$ only via $g^{-1}g'$ for $g' \in G_f$, and $g^{-1}g'$ is invariant to the group action $g \mapsto ug, g' \mapsto ug'$. This is enough to show that $\Phi$ satisfies the above condition for equivariance. See \autoref{sec:proof} for full proof.

\textbf{Generalisation to infinite $\bm{G_f}$}. For Lie Groups, $G_f$ is usually infinite (it need not be if $H$ is discrete e.g. for $T(n)$ acting on $\mathbb{R}^n$, we have $H=\{e\}$ hence $G_f$ is finite). To deal with this case we resort to Monte Carlo (MC) estimation to approximate the integral in \verb!LieSelfAttention!, following the approach of \citet{finzi2020generalizing}: 
\begin{enumerate}
    \vspace{-0.5em}
    \item Replace $G_f \triangleq \cup_{i=1}^n s(x_i)H $ with a finite subset $\hat{G}_f \triangleq \cup_{i=1}^n s(x_i) \hat{H}$ where $\hat{H}$ is a finite subset of $H$ sampled uniformly. We refer to $|\hat{H}|$ as the number of \textit{lift samples}.
    \vspace{-0.5em}
    \item (Optional, for computational efficiency) Further replace $\hat{G}_f$ with uniform samples from the neighbourhood $\mathtt{nbhd}_{\eta}(g) \triangleq \{g' \in \hat{G}_f: d(g,g') \leq \eta\}$ for some threshold $\eta$ where distance is measured by the log map $d(g,g')=||\nu[\log(g^{-1}g')]||$ (c.f. \autoref{sec:logmap}).
    \vspace{-0.5em}
\end{enumerate}
See \autoref{fig:lifting_sampling} for a visualisation. Due to MC estimation we now have equivariance in expectation as \citet{finzi2020generalizing}. For sampling within the neighbourhood, we can show that the resulting \verb!LieSelfAttention! is still equivariant in expectation given that the distance is a function of $g^{-1}g'$ (c.f. \autoref{sec:proof}).

\section{Related Work}
\textbf{Equivariant maps with/without lifting} Equivariant neural networks can be broadly categorised by whether the input spatial data is \textit{lifted} onto the space of functions on group $G$ or not. Without lifting, the equivariant map is defined between the space of functions/features on the homogeneous input space $X$, with equivariance imposing a constraint on the parameterisation of the convolutional kernel or attention module \citep{cohen2016steerable, worrall2017harmonic, thomas2018tensor, kondor2018clebsch, weiler2018learning, weiler20183d, weiler2019general, esteves2020spin, fuchs2020se}. In the case of convolutions, the kernel is expressed using a basis of equivariant functions such as circular or spherical harmonics. However with lifting, the equivariant map is defined between the space of functions/features on $G$, and aforementioned constraints on the convolutional kernel or attention module are relaxed at the cost of an increased dimensionality of the input to the neural network \citep{cohen2016group, cohen2018spherical, esteves2018learning, finzi2020generalizing, bekkers2020b, Romero2020CoAttentive, romero2020attentive, hoogeboom2018hexaconv}. Our method also uses lifting to define equivariant self-attention.

\textbf{Equivariant self-attention} Most of the above works use equivariant convolutions as the core building block of their equivariant module, drawing from the result that bounded linear operators are group equivariant if and only if they are convolutions \citep{kondor2018generalization, cohen2019general, bekkers2020b}. Such convolutions are used with pointwise non-linearities (applied independently to the features at each spatial location/group element) to form expressive equivariant maps. Exceptions to this are \citet{romero2020attentive} and \citet{fuchs2020se} that explore equivariant attentive convolutions, reweighing convolutional kernels with attention weights. This gives non-linear equivariant maps with non-linear interactions across spatial locations/group elements. Instead, our work removes convolutions and investigates the use of equivariant self-attention only, inspired by works that use stand-alone self-attention on images to achieve competitive performance to convolutions \citep{ParRamVas2019, dosovitskiy2020image}. Furthermore,  \citet{romero2020attentive} focus on image applications (hence scalability) and discrete groups (p4, p4m), and \citet{fuchs2020se} focus on 3D point cloud applications and the $SE(3)$ group with irreducible representations acting on functions on $\mathcal{X}$. Instead we use regular representations actingon functions on $G$, and give a general method for Lie groups acting on homogeneous spaces, with a wide range of applications from dealing with point cloud data to modelling Hamiltonian dynamics of particles. This is very much in the spirit of \citet{finzi2020generalizing}, except for self-attention instead of convolutions. In concurrent work, \citet{romero2020group} describe group equivariant self-attention also using lifting and regular representations. Their analogue of location-based attention are group invariant positional encodings. The main difference between the two works is that \citet{romero2020group} specify methodology for discrete groups applied to image classification only and it is not clear how to extend their approach to Lie groups. In contrast, our method provides a general formula for (unimodular) Lie groups and their discrete subgroups for the aforementioned applications.

\section{Experiments}
We consider three different tasks that have certain symmetries, highlighting the benefits of the \verb!LieTransformer!: \begin{inparaenum}[(1)] \item Counting shapes in 2D point cloud of constellations \item Molecular property regression and \item Modelling particle trajectories under Hamiltonian dynamics.\footnote{The code for our experiments is available at: \url{https://github.com/oxcsml/lie-transformer}}
\end{inparaenum}

\subsection{Counting Shapes in 2D Point Clouds} \label{sec:constellation}
\renewcommand{\arraystretch}{1.2}
\newcommand{\shaderowlieconv}{\rowcolor{DarkOrange!10}}
\newcommand{\shaderowlietran}{\rowcolor{ForestGreen!10}}
\begin{table*}[t!]
    \setlength\tabcolsep{3pt}
    \centering
    \small
\begin{tabular}{lcccccc} 
    \toprule[1pt]
    \multirow{1}{*}{\bfseries{Training data}} & 
     \multicolumn{1}{c}{$D_{\text{train}}$} &
 \multicolumn{1}{c}{$D_{\text{train}}$} &
 \multicolumn{1}{c}{$D_{\text{train}}$} &
 \multicolumn{1}{c}{$D_{\text{train}}^{T2}$} &
 \multicolumn{1}{c}{$D_{\text{train}}^{T2}$} &
 \multicolumn{1}{c}{$D_{\text{train}}^{SE2}$}\\ 
  \multicolumn{1}{l}{\bfseries{Test data}}& \multicolumn{1}{c}{$D_{\text{test}}$} & \multicolumn{1}{c}{$D_{\text{test}}^{T2}$} & \multicolumn{1}{c}{$D_{\text{test}}^{SE2}$} & \multicolumn{1}{c}{$D_{\text{test}}^{T2}$} & 
  \multicolumn{1}{c}{$D_{\text{test}}^{SE2}$} & 
  \multicolumn{1}{c}{$D_{\text{test}}^{SE2}$} \\
 \cmidrule(lr){1-7}
    
    SetTransformer&0.58 $\pm$ 0.07 & 0.44 $\pm$ 0.02 & 0.44 $\pm$ 0.02 & 0.61 $\pm$ 0.02 & 0.51 $\pm$ 0.01 & 0.55 $\pm$ 0.01 \\
    \shaderowlietran LieTransformer-T2 & \textbf{0.75} $\pm$ 0.03 & \textbf{0.75} $\pm$ 0.03 & 0.63 $\pm$ 0.06 & \textbf{0.75} $\pm$ 0.03 &0.63 $\pm$ 0.06 & 0.70 $\pm$  0.03 \\
    \shaderowlietran LieTransformer-SE2 & 0.71 $\pm$ 0.01 &0.71 $\pm$ 0.01 &  \textbf{0.69} $\pm$ 0.02 & 0.71 $\pm$ 0.01  &  \textbf{0.69} $\pm$ 0.02 &\textbf{0.72} $\pm$ 0.04  \\
    \bottomrule
  \end{tabular}
    \vspace{-1em}
    \caption{Mean and standard deviation of test accuracies on the shape counting task at convergence (over 8 random initialisations).}
    \label{tab:constellation_results}
\end{table*}
We first consider the toy, synthetic task of counting shapes in a 2D point cloud $\{x_1, x_2, ..., x_K\}$ of constellations \citep{kosiorek2019stacked}, mainly to check that \verb!LieTransformer! has the correct invariance properties. We use $\mathtt{f}_i=1$ for all points. Each example consists of points in the plane that form the vertices of a pattern. There are four types of patterns: triangles, squares, pentagons and the `L' shape, with varying sizes, orientation, and number of instances per pattern (see \autoref{fig:invariance-error} (right)). The task is to classify the number of instances of each pattern, hence is invariant to 2D roto-translations $SE(2)$.

We first create a fixed training set $D_{\text{train}}$ and test set $D_{\text{test}}$ of size 10,000 and 1,000 respectively. We then create augmented test sets $D^{T2}_{\text{test}}$ and $D^{SE2}_{\text{test}}$ that are copies of $D_{\text{test}}$ with arbitrary transformations in $T(2)$ and $SE(2)$ respectively. In \autoref{tab:constellation_results}, we evaluate the test accuracy of \verb!LieTransformer! at convergence with and without data augmentation during training time -- $D^{T2}_{\text{train}}$ and $D^{SE2}_{\text{train}}$ indicate random $T(2)$ and $SE(2)$ augmentations respectively to each batch of $D_{\text{train}}$ at every training iteration. 
We evaluate the test performance of \verb!LieTransformer-T2! and \verb!LieTransformer-SE2! that are invariant to $T(2)$ and $SE(2)$ respectively, against the baseline \verb!SetTransformer! \citep{LeeLeeKim2019a}, a Transformer-based model that is permutation invariant, but not invariant to rotations nor translations. We use a similar number of parameters for each model. See Appendix \ref{sec:constellation_setup} for further details on the setup.

\begin{figure}[h!]
     \centering
     \centering

     \raisebox{-0.5\height}{\includegraphics[width=0.48\linewidth]{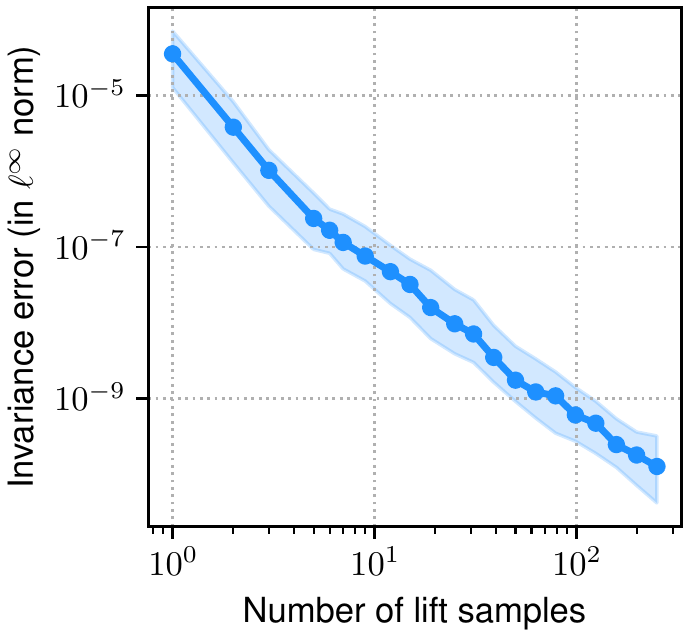}}
     \hspace{1em}
      \raisebox{-0.41\height}{\includegraphics[width=0.44\linewidth]{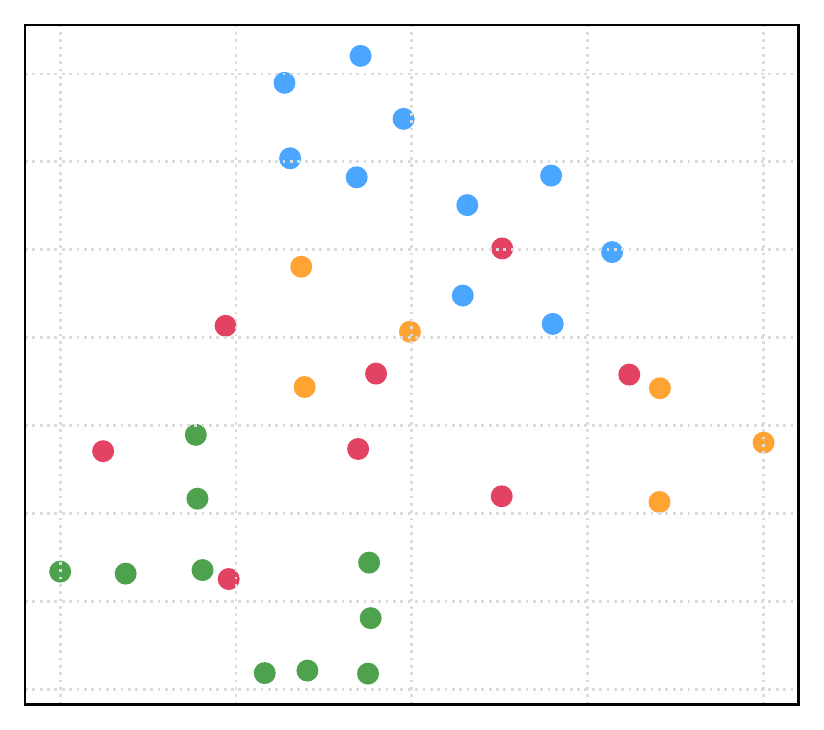}}
     \caption{(Left) $SE(2)$ invariance error on output logits vs. number of lift samples for a single layer \texttt{LieTransformer-SE2}. Plot shows median and interquartile range across 100 runs, randomizing over model seed, input example and transformation applied to input. (Right) An example 2D point cloud from $D_{\text{train}}$. Each colour corresponds to a different pattern.} \label{fig:invariance-error}
\end{figure}

Note that the test accuracy of \verb!LieTransformer-T2! and \verb!LieTransformer-SE2! remains unchanged when the train/test set is augmented with $T(2)$. For \verb!LieTransformer-SE2!, this is not quite true for $SE(2)$ augmentations because the model is only $SE(2)$ equivariant in expectation and not exactly equivariant given a finite number of lifts samples ($|\hat{H}|$). However the changes in accuracy for $SE(2)$ augmentation are much smaller compared to \verb!LieTransformer-T2!. The test accuracy of \verb!SetTransformer!, the non-invariant baseline, is always lower than \verb!LieTransformer!. Note that \verb!LieTransformer-T2! does slightly better than \verb!LieTransformer-SE2! on $D_{\text{test}}$ and $D^{T2}_{\text{test}}$. We suspect that the variance in the sampling of the lifting layer for \verb!LieTransformer-SE2! is making optimisation more difficult, and will continue to explore these results.

In \autoref{fig:invariance-error} (left), we report the equivariance error of \verb!LieTransformer-SE2! when increasing the number of lift samples ($|\hat{H}|$) used in the Monte Carlo approximation of \verb!LieSelfAttention!. As expected the invariance error decreases monotonically with the number of lift samples, and already with 3 lift samples, the error is small ($\approx 10^{-6}$).


\renewcommand{\arraystretch}{1.2}
\definecolor{fu-blue}{RGB}{0, 51, 102} 
\begin{table*}[t]
    \setlength\tabcolsep{3pt}
    \centering
    \small
    \begin{tabular}{l|rrrrrrrrrrrr}
       \toprule[1pt]
        Task & \( \alpha \) & \(\Delta \epsilon\) & \( \epsilon_{\text{HOMO}}\) &  \( \epsilon_{\text{LUMO}} \) & \( \mu \) & \( C_{\nu} \) & \( G \) & \( H \) & \( R^2 \) & \( U \) & \( U_0 \) & ZPVE \\
        Units & bohr$^3$ & meV & meV & meV & D & cal/mol K & meV & meV & bohr$^{2}$ & meV & meV & meV \\ \midrule[1pt]
        WaveScatt \citep{hirn2017wavelet} & $.160$ & $118$ & $85$ & $76$ & $.340$ & $.049$ & $-$ & $-$ & $-$ & $-$ & $-$ & $-$ \\ 
        NMP \citep{gilmer2017neural} & $\mathbf{.092}$ & $69$ & $43$ & $38$ & $\underline{\mathbf{.030}}$ & $.040$ & $19$ & $17$ & $.180$ & $20$ & $20$ & $\mathbf{1.50}$ \\ 
        SchNet \citep{schutt2017schnet} & $.235$ & $\mathbf{63}$ & $\mathbf{41}$ & $\mathbf{34}$ & $.033$ & $\mathbf{.033}$ & $\mathbf{14}$ & $\mathbf{14}$ & $\underline{\mathbf{.073}}$ & $\mathbf{19}$ & $\mathbf{14}$ & $1.70$ \\ \midrule
        Cormorant \citep{anderson2019cormorant} & $.085$ & $61$ & $34$ & $38$ & $.038$ & $.026$ & $20$ & $21$ & $.961$ & $21$ & $22$ & $2.03$ \\
        DimeNet++ \citep{klicpera2020fast} $^*$ & $\underline{\mathbf{.049}}$ & $\underline{\mathbf{34}}$ & $\underline{\mathbf{26}}$ & $\underline{\mathbf{20}}$ & $\mathbf{.033}$ & $\underline{\mathbf{.024}}$ & $\underline{\mathbf{8}}$ & $\underline{\mathbf{7}}$ & $.387$ & $\underline{\mathbf{7}}$ & $\underline{\mathbf{7}}$  & $\underline{\mathbf{1.23}}$ \\ 

        L1Net \citep{miller2020relevance} & $.088$ & $68$ & $45$ & $35$ & $.043$ & $.031$ & $14$ & $14$ & $\mathbf{.354}$ & $14$ & $13$ & $1.56$ \\
        \midrule
        TFN \citep{thomas2018tensor} & $.223$ & $58$ & $40$ & $38$ & $.064$ & $.101$ & $-$ & $-$ & $-$ & $-$ & $-$ & $-$ \\
        SE3-Transformer \citep{fuchs2020se} & $.148$ & $53$ & $36$ & $33$ & $.053$ & $.057$ & $-$ & $-$ & $-$ & $-$ & $-$ & $-$ \\
        \shaderowlieconv LieConv-T3 \citep{finzi2020generalizing} $^\dagger$ & $.125$ & $60$ & $36$ & $32$ & $.057$ & $.046$ & $35$ & $37$ & $1.54$ & $36$ & $35$ & $3.62$ \\
        \shaderowlieconv LieConv-T3 + SO3 Aug \citep{finzi2020generalizing} & $.084$ & $49$ & $30$ & $\mathbf{25}$ & $\mathbf{.032}$ & $.038$ & $22$ & $24$ & $.800$ & $19$ & $19$ & $2.28$ \\
        \shaderowlieconv LieConv-SE3 \citep{finzi2020generalizing}$^\dagger$ & $.097$ & $\mathbf{45}$ & $\mathbf{27}$ & $\mathbf{25}$ & $.039$ & $.041$ & $39$ & $46$ & $2.18$ & $49$ & $48$ & $3.27$ \\
        \shaderowlieconv LieConv-SE3 + SO3 Aug \citep{finzi2020generalizing}$^\dagger$ & $.088$ & $\mathbf{45}$ & $\mathbf{27}$ & $\mathbf{25}$ & $.038$ & $.043$ & $47$ & $46$ & $2.12$ & $44$ & $45$ & $3.25$ \\
        \shaderowlietran LieTransformer-T3 (Us) & $.179$ & $67$ & $47$ & $37$ & $.063$ & $.046$ & $27$ & $29$ & $.717$ & $27$ & $28$ & $2.75$ \\
        \shaderowlietran LieTransformer-T3 + SO3 Aug (Us) & $\mathbf{.082}$ & $51$ & $33$ & $27$ & $.041$ & $\mathbf{.035}$ & $\mathbf{19}$ & $\mathbf{17}$ & $\mathbf{.448}$ & $\mathbf{16}$ & $\mathbf{17}$ & $\mathbf{2.10}$\\
        \shaderowlietran LieTransformer-SE3 (Us)& $.104$  & $52$ & $33$ & $29$ & $.061$ & $.041$ & $23$ & $27$ & $2.29$ & $26$ & $26$ & $3.55$ \\
        \shaderowlietran LieTransformer-SE3 + SO3 Aug (Us) & $.105$  & $52$ & $33$ & $29$ & $.062$ & $.041$ & $22$ & $25$ & $2.31$ & $24$ & $25$ & $3.67$ \\
        \bottomrule[1pt]
    \end{tabular}
    \caption{QM9 molecular property prediction mean absolute error. \iffalse Upper section of the table are non-equivariant models designed specifically for molecular property prediction, middle section are equivariant models designed specifically for molecular property prediction, lower section are general purpose equivariant models. \fi Bold indicates best performance in a given section, underlined indicates best overall performance. $^*$These results are from our own runs of the Dimenet++ model. The original paper used different train/valid/test splits to the other papers listed here. $^\dagger$These results are from our owns runs of LieConv as these ablations were not present in the original paper.}
    \label{tab:QM9_results}
    \vspace{-1.5em}
\end{table*}

\subsection{QM9: Molecular Property Regression}

We apply the \verb!LieTransformer! to the QM9 molecule property prediction task \citep{ruddigkeit_enumeration_2012, ramakrishnan2014quantum}. This dataset consists of 133,885 small inorganic molecules described by the location and charge of each atom in the molecule, along with the bonding structure of the molecule. The dataset includes 19 properties of each molecule, such as various rotational constants, energies and enthalpies, and 12 of these are used as regression tasks. We expect these molecular properties to be invariant to 3D roto-translations $SE(3)$. We follow the customary practice of performing hyperparameter search on the $\epsilon_{\text{HOMO}}$ task and use the same hyperparameters for training on the other 11 tasks. Further details of the exact experimental setup can be found in Appendix \ref{sec:QM9_setup}.

We trained four variants of both \verb!LieTransformer! and \verb!LieConv!, namely the $T(3)$ and $SE(3)$ invariant models with and without $SO(3)$ (rotation) data augmentation. We set $x_i$ to be the atomic position and $\mathtt{f}_i$ to be the charge.
Table \ref{tab:QM9_results} shows the test error of all models and baselines on the 12 tasks. The table is divided into 3 sections. \textbf{Upper}: non-invariant models specifically designed for the QM9 task. \textbf{Middle}: invariant models specifically designed for the QM9 task. \textbf{Lower}: invariant models that are general-purpose. We show very competitive results, and perform best of general-purpose models on 8/12 tasks. In particular when comparing against \verb!LieConv!, we see better performance on the majority of tasks, suggesting that the attention framework is better suited to these tasks than convolutions.

As expected for both \verb!LieTransformer! and \verb!LieConv!, the $SE(3)$ models tend to outperform the $T(3)$ models without $SO(3)$ data augmentation (on 10/12 tasks and 7/12 tasks respectively), showing that being invariant to rotations improves generalisation. Moreover the $SE(3)$ models perform similarly with and without augmentation, whereas the $T(3)$ models greatly benefit from augmentation, showing evidence that the $SE(3)$ models are indeed invariant to rotations. However the $T(3)$ models with augmentation outperform the $SE(3)$ counterparts on most tasks for both \verb!LieTransformer-SE3! and \verb!LieConv-SE3!. As for the experiments in Section \ref{sec:constellation}, we suspect that the variance in the sampling of the lifting layer of $SE(3)$ models, along with the $SE(3)$ log-map (\autoref{sec:logmap}) in the location attention is making optimisation more difficult, and plan to continue investigating the source of this discrepancy in performance. Note however that \verb!LieTransformer-SE3! and \verb!LieConv-SE3! tend to outperform the irreducible representation (irrep) based $SE(3)$-Transformer and TFN. This can be seen as further evidence that regular representation approaches tend to outperform irrep approaches, in line with the empirical observations of \citet{weiler2019general}.

\subsection{Modelling Particle Trajectories with Hamiltonian Dynamics}
\label{sec:hamiltonian}
We also apply the \verb!LieTransformer! to a physics simulation task in the context of Hamiltonian dynamics, a formalism for describing the evolution of a physical system using a single scalar function $H(q,p)$, called the Hamiltonian. 

We consider the case of $n$ particles in $d$ dimensions, writing the position $\mathbf{q} \in \mathbb{R}^{nd}$ and momentum $\mathbf{p} \in \mathbb{R}^{nd}$ of all particles as a single state $\mathbf{z} = (\mathbf{q}, \mathbf{p})$. The Hamiltonian $H: \mathbb{R}^{2nd} \to \mathbb{R}$ takes as input $\mathbf{z}$ and returns its total (potential plus kinetic) energy. The time evolution of the particles is then given by \textit{Hamilton's equations}:
\begin{equation} \label{eq:hamiltons-eqs}
\frac{\mathrm{d}\mathbf{q}}{\mathrm{d}t} = \frac{\partial H}{\partial \mathbf{p}}, ~~
\frac{\mathrm{d}\mathbf{p}}{\mathrm{d}t} = -\frac{\partial H}{\partial \mathbf{q}}.
\end{equation}
Several recent works have shown that modelling physical systems by learning its Hamiltonian significantly outperforms approaches that learn the dynamics directly \citep{greydanus2019hamiltonian, sanchez2019hamiltonian, zhong2019symplectic, finzi2020generalizing}. Specifically, we can parameterise the Hamiltonian of a system by a neural network $H_\theta$ that is learned by ensuring that trajectories from the ground truth and learned system are close to each other. Given a learned $H_\theta$, we can simulate the system for $T$ timesteps by solving equation (\ref{eq:hamiltons-eqs}) with a numerical ODE solver to obtain a trajectory $\{\hat{\mathbf{z}}_t(\theta)\}_{t=1}^{T}$ and minimize the $\ell^2$-norm between this trajectory and the ground truth $\{\mathbf{z}_t\}_{t=1}^{T}$.

However we know a-priori that such physical systems have symmetries, namely conserved quantities such as linear and angular momentum. A notable result is Noether's theorem \citep{noether1971invariant}, which states that the system has a conserved quantity if and only if the Hamiltonian is group-invariant. For example, translation invariance of the Hamiltonian implies conservation of momentum and rotation invariance implies conservation of angular momentum. Hence in our experiments, we parameterise the Hamiltonian $H_\theta$ by a \verb!LieTransformer! and endow it with the symmetries corresponding to the conservation laws of the physical system we are modelling. We test our model on the spring dynamics task proposed in \citet{sanchez2019hamiltonian} -- we consider a system of $6$ particles with randomly sampled massses in 2D, where each particle connected to all others by springs. This system conserves both linear and angular momentum, so the ground truth Hamiltonian will be both translationally and rotationally invariant, that is, $SE(2)$-invariant. We simulate this system for 500 timesteps from random initial conditions and use random subsets of length 5 from these roll-outs to train the model (see Appendix \ref{sec:Hamiltonian_setup} for full experimental details).

We compare our method to different parameterisations of $H_\theta$, namely Fully-connected network \citep{chen2018neural}, Graph Network \citep{sanchez2019hamiltonian} and \verb!LieConv!. Only \verb!LieTransformer! and \verb!LieConv! incorporate invariance. In Figures \ref{fig:hamiltonian-data-efficiency}, \ref{fig:hamiltonian-1e4-rollout}, and \ref{fig:hamiltonian-trajectories}, we use \verb!LieTransformer-T2! and \verb!LieConv-T2! since \citet{finzi2020generalizing} report that there are numerical instabilities for \verb!LieConv-SE2! on this task, due to which \verb!LieConv-T2! is their default model and performs the best. However in Figure \ref{fig:hamiltonian-parameter-efficiency-5-step}, we also consider $SE(2)$-invariant versions of both models with modifications to the lifting procedure, which fixed the instabilities as outlined in Appendix \ref{sec:logmap}.

\begin{figure}[t]
    \centering
    \includegraphics[width=.49\textwidth]{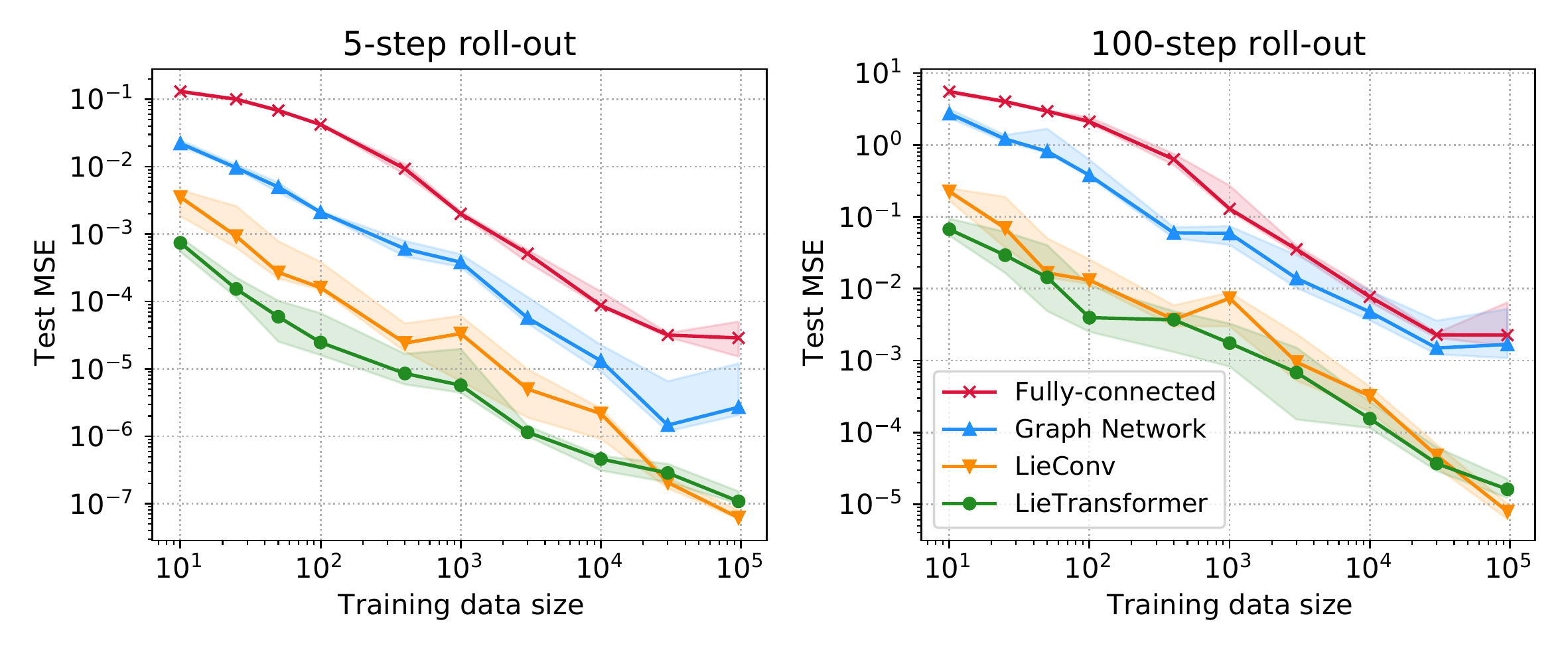}
    \caption{Data efficiency on Hamiltonian spring dynamics. All models are trained using 5-step roll-outs, with test performance on 5-step (left) and 100-step (right) roll-outs. Plots show median MSE and interquartile range (IQR) across 10 random seeds.}
    \label{fig:hamiltonian-data-efficiency}
    \vspace{-10pt}
\end{figure}


Figure \ref{fig:hamiltonian-data-efficiency} compares the performance of all methods as a function of the number of training examples. \verb!LieTransformer! is highly data-efficient: the inductive bias from the symmetries of the Hamiltonian allow us to accurately learn the dynamics even from a small training set. Our method consistently outperforms non-invariant methods (fully-connected and graph networks), typically by 1-3 orders of magnitude. Furthermore, our method outperforms \verb!LieConv! for most data sizes except the largest sizes where the errors are similar, suggesting that the attention framework more suited for this task.

\begin{wrapfigure}{r}{0.25\textwidth}
    \centering
    \includegraphics[width=0.25\textwidth]{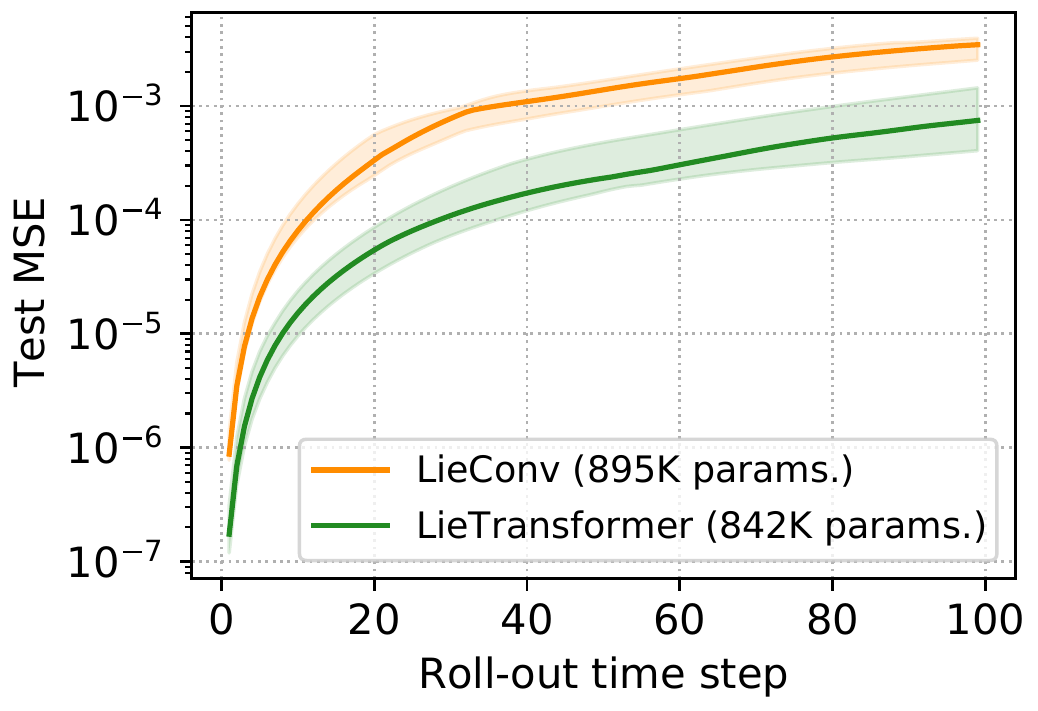}
    \caption{Test error vs. roll-out time step. Plots show median MSE and IQR across 10 random seeds.}
    \label{fig:hamiltonian-1e4-rollout}
\end{wrapfigure}

Figure \ref{fig:hamiltonian-1e4-rollout} shows the test error as function of the roll-out time step for a training data size of 10,000 (corresponding plots for other training data sizes are included in Appendix \ref{sec:hamiltonian_extra}). Here we show that the \verb!LieTransformer! shows better generalisation than \verb!LieConv! across all roll-out lengths, the error being low ($<10^{-3}$) for 100 step-roll-outs even though we only train on 5-step roll-outs. We also include example trajectories of our model in Figure \ref{fig:hamiltonian-trajectories} (more examples can be found in the appendix, including ones where \verb!LieConv! performs better than \verb!LieTransformer!) illustrating the accuracy of our model on this task.

\begin{figure}[t]
    \centering
    \includegraphics[width=.4\textwidth]{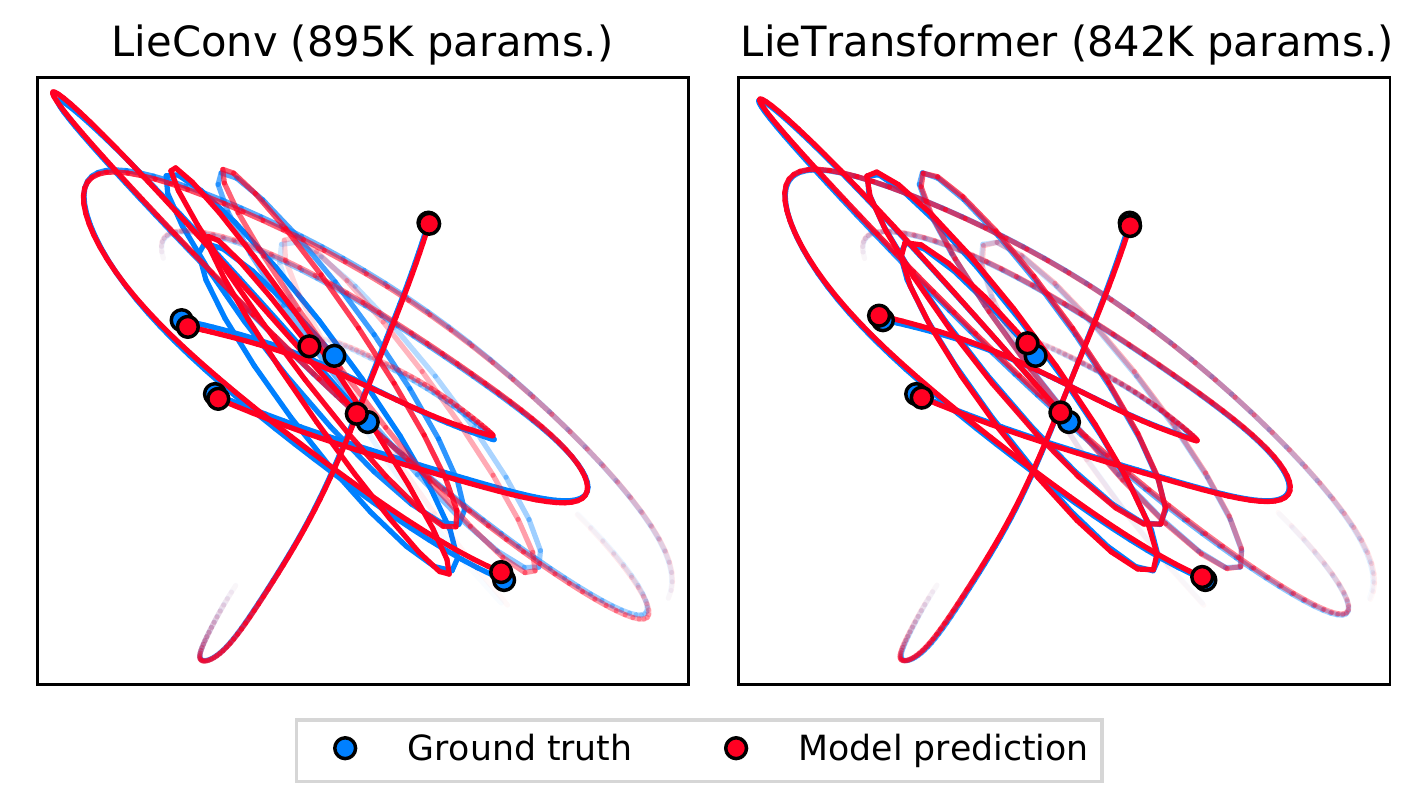}
    \caption{Example trajectory predictions on the spring dynamics task. LieTransformer closely follows the ground truth while LieConv diverges from the ground truth at later timesteps.}
    \label{fig:hamiltonian-trajectories}
    \vspace{-10pt}
\end{figure}

Lastly, Figure \ref{fig:hamiltonian-parameter-efficiency-5-step} compares \verb!LieConv! and \verb!LieTransformer! for different model sizes (number of parameters) and equivariance groups. We first note \verb!LieTransformer! outperforms \verb!LieConv! given a fixed model size and group. For $T(2)$-invariant models, our method benefits from a larger model, whereas \verb!LieConv! deteriorates (\verb!LieConv-T(2) (895K)! is their default architecture on this task). However, for both methods, the $SE(2)$-invariant models perform at par with or better than their $T(2)$-invariant counterparts despite having smaller model sizes. In particular, \verb!LieTransformer-SE(2) (139K)! outperforms all other models in this comparison despite having the smallest number of parameters, which highlights the advantage of incorporating the correct task symmetries into the architecture and the attention framework. Overall, we have shown that our model is suitable for use in a neural ODE setting that requires equivariant drift functions. 

\section{Limitations and Future Work}

From the algorithmic perspective, \verb!LieTransformer! shares the weakness of \verb!LieConv! in being memory-expensive ($O(|\hat{G}_f||\mathtt{nbhd}_{\eta}|)$ memory cost (\autoref{sec:memory}) due to:  1. The lifting procedure that increases the number of inputs by $|\hat{H}|$, and 2. Quadratic complexity in the number of inputs from having to compute the kernel value at each pair of inputs. Although the first is a weakness shared by all lifting-based equivariant neural networks, the second can be addressed by incorporating works that study efficient variants of self-attention \citep{wang2020linformer, kitaev2020reformer, zaheer2020big, katharopoulos2020transformers}. An alternative is to incorporate information about pairs of inputs (such as bonding information for the QM9 task) as masking in self-attention (c.f. Appendix \ref{sec:QM9_setup}).

From the methodological perspective, a key weakness of the \verb!LieTransformer! that is also shared with \verb!LieConv! is its approximate equivariance due to MC estimation of the integral in \verb!LieSelfAttention! for the case where $H$ is infinite. The aforementioned directions for memory-efficiency can help to reduce the approximation error by allowing to use more lift samples ($|\hat{H}|$). Other directions include incorporating the notion of \emph{steerability} \citep{cohen2016steerable} to deal with vector fields in an equivariant manner (given inputs $(x_i,\mathtt{f}_i)$, the group acts non-trivially on $\mathtt{f}_i$ as well as $x_i$), and extending to non-homogeneous input spaces as outlined in \citet{finzi2020generalizing}.

\begin{figure}[t]
    \centering
    \includegraphics[width=.46\textwidth, trim=0mm 0mm 0mm 16.4mm, clip]{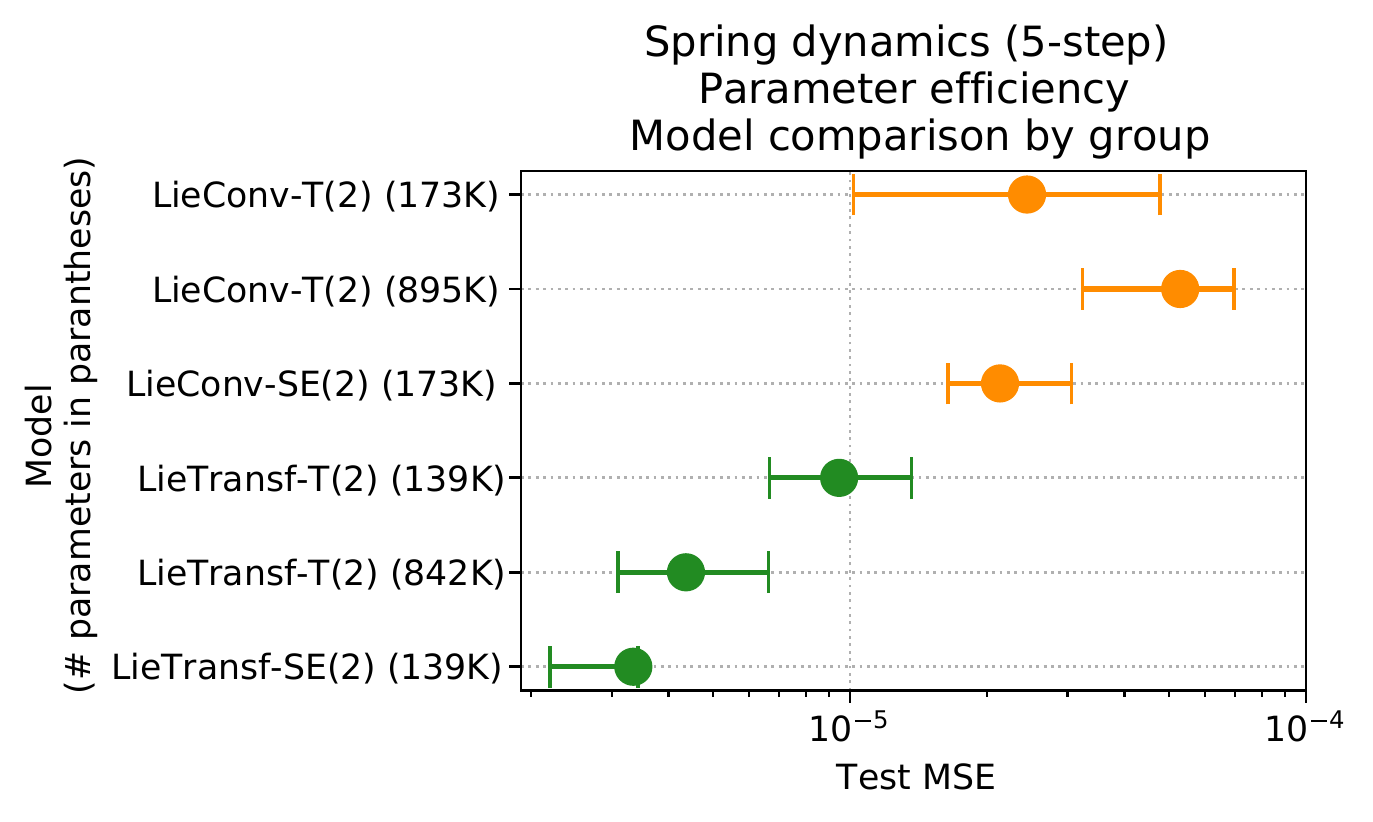}
    \caption{Median test MSE and IQR on 5-step trajectories, across 5 random seeds. Results for 100-step trajectories in Figure \ref{fig:hamiltonian-parameter-efficiency-100-step}.}
    \vspace{-4mm}
    \label{fig:hamiltonian-parameter-efficiency-5-step}
\end{figure}

\subsubsection*{Acknowledgements}
The authors would like to thank Adam R.Kosiorek for setting up the initial codebase at the beginning of the project, and David W. Romero \& Jean-Baptiste Cordonnier for useful discussions. Michael is supported by the EPSRC Centre for Doctoral Training in Modern Statistics and Statistical Machine Learning (EP/S023151/1).
Charline acknowledges funding from the EPSRC grant agreement no. EP/N509711/1. Sheheryar wishes to acknowledge support from Aker Scholarship. Emilien acknowledges support of his PhD funding from Google DeepMind. Yee Whye Teh’s research leading to these results has received funding from the European Research Council under the European Union’s Seventh Framework Programme (FP7/2007-2013) ERC grant agreement no. 617071.

We would also like to thank
the Python community~\citep{van1995python,oliphant2007python} for developing
the tools that enabled this work, including Pytorch \citep{paszke2017automatic},
{NumPy}~\citep{oliphant2006guide,walt2011numpy, harris2020array},
{SciPy}~\citep{jones2001scipy}, and
{Matplotlib}~\citep{hunter2007matplotlib}.



\bibliography{icml}

\begin{thebibliography}{65}
\providecommand{\natexlab}[1]{#1}
\providecommand{\url}[1]{\texttt{#1}}
\expandafter\ifx\csname urlstyle\endcsname\relax
  \providecommand{\doi}[1]{doi: #1}\else
  \providecommand{\doi}{doi: \begingroup \urlstyle{rm}\Url}\fi

\bibitem[Anderson et~al.(2019)Anderson, Hy, and Kondor]{anderson2019cormorant}
Anderson, B., Hy, T.~S., and Kondor, R.
\newblock Cormorant: Covariant molecular neural networks.
\newblock In \emph{NeurIPS}, 2019.

\bibitem[Ba et~al.(2016)Ba, Kiros, and Hinton]{ba2016layer}
Ba, J.~L., Kiros, J.~R., and Hinton, G.~E.
\newblock Layer normalization.
\newblock \emph{arXiv preprint arXiv:1607.06450}, 2016.

\bibitem[Bekkers(2020)]{bekkers2020b}
Bekkers, E.~J.
\newblock {B}-spline {CNNs} on {L}ie groups.
\newblock In \emph{ICLR}, 2020.

\bibitem[Bloem-Reddy \& Teh(2020)Bloem-Reddy and Teh]{bloem2020probabilistic}
Bloem-Reddy, B. and Teh, Y.~W.
\newblock Probabilistic symmetries and invariant neural networks.
\newblock \emph{Journal of Machine Learning Research}, 21\penalty0
  (90):\penalty0 1--61, 2020.

\bibitem[Brown et~al.(2020)Brown, Mann, Ryder, Subbiah, Kaplan, Dhariwal,
  Neelakantan, Shyam, Sastry, Askell, et~al.]{brown2020language}
Brown, T.~B., Mann, B., Ryder, N., Subbiah, M., Kaplan, J., Dhariwal, P.,
  Neelakantan, A., Shyam, P., Sastry, G., Askell, A., et~al.
\newblock Language models are few-shot learners.
\newblock \emph{arXiv preprint arXiv:2005.14165}, 2020.

\bibitem[Chen et~al.(2018)Chen, Rubanova, Bettencourt, and
  Duvenaud]{chen2018neural}
Chen, R.~T., Rubanova, Y., Bettencourt, J., and Duvenaud, D.~K.
\newblock Neural ordinary differential equations.
\newblock In \emph{NeurIPS}, 2018.

\bibitem[Cohen \& Welling(2016)Cohen and Welling]{cohen2016group}
Cohen, T. and Welling, M.
\newblock Group equivariant convolutional networks.
\newblock In \emph{ICML}, 2016.

\bibitem[Cohen \& Welling(2017)Cohen and Welling]{cohen2016steerable}
Cohen, T.~S. and Welling, M.
\newblock Steerable {CNNs}.
\newblock In \emph{ICLR}, 2017.

\bibitem[Cohen et~al.(2018)Cohen, Geiger, K{\"o}hler, and
  Welling]{cohen2018spherical}
Cohen, T.~S., Geiger, M., K{\"o}hler, J., and Welling, M.
\newblock Spherical {CNNs}.
\newblock In \emph{ICLR}, 2018.

\bibitem[Cohen et~al.(2019)Cohen, Geiger, and Weiler]{cohen2019general}
Cohen, T.~S., Geiger, M., and Weiler, M.
\newblock A general theory of equivariant {CNNs} on homogeneous spaces.
\newblock In \emph{NeurIPS}, 2019.

\bibitem[Dosovitskiy et~al.(2020)Dosovitskiy, Beyer, Kolesnikov, Weissenborn,
  Zhai, Unterthiner, Dehghani, Minderer, Heigold, Gelly,
  et~al.]{dosovitskiy2020image}
Dosovitskiy, A., Beyer, L., Kolesnikov, A., Weissenborn, D., Zhai, X.,
  Unterthiner, T., Dehghani, M., Minderer, M., Heigold, G., Gelly, S., et~al.
\newblock An image is worth 16x16 words: Transformers for image recognition at
  scale.
\newblock \emph{arXiv preprint arXiv:2010.11929}, 2020.

\bibitem[Esteves et~al.(2018)Esteves, Allen-Blanchette, Makadia, and
  Daniilidis]{esteves2018learning}
Esteves, C., Allen-Blanchette, C., Makadia, A., and Daniilidis, K.
\newblock Learning {SO}(3) equivariant representations with spherical {CNNs}.
\newblock In \emph{ECCV}, 2018.

\bibitem[Esteves et~al.(2020)Esteves, Makadia, and Daniilidis]{esteves2020spin}
Esteves, C., Makadia, A., and Daniilidis, K.
\newblock Spin-weighted spherical {CNNs}.
\newblock In \emph{NeurIPS}, 2020.

\bibitem[Finzi et~al.(2020)Finzi, Stanton, Izmailov, and
  Wilson]{finzi2020generalizing}
Finzi, M., Stanton, S., Izmailov, P., and Wilson, A.~G.
\newblock Generalizing convolutional neural networks for equivariance to {L}ie
  groups on arbitrary continuous data.
\newblock In \emph{ICML}, 2020.

\bibitem[Fuchs et~al.(2020)Fuchs, Worrall, Fischer, and Welling]{fuchs2020se}
Fuchs, F.~B., Worrall, D.~E., Fischer, V., and Welling, M.
\newblock {SE}(3)-{T}ransformers: 3{D} roto-translation equivariant attention
  networks.
\newblock In \emph{NeurIPS}, 2020.

\bibitem[Gilmer et~al.(2017)Gilmer, Schoenholz, Riley, Vinyals, and
  Dahl]{gilmer2017neural}
Gilmer, J., Schoenholz, S.~S., Riley, P.~F., Vinyals, O., and Dahl, G.~E.
\newblock Neural message passing for quantum chemistry.
\newblock 2017.

\bibitem[Graves \& Jaitly(2014)Graves and Jaitly]{graves2014towards}
Graves, A. and Jaitly, N.
\newblock Towards end-to-end speech recognition with recurrent neural networks.
\newblock In \emph{ICML}, 2014.

\bibitem[Greydanus et~al.(2019)Greydanus, Dzamba, and
  Yosinski]{greydanus2019hamiltonian}
Greydanus, S., Dzamba, M., and Yosinski, J.
\newblock {H}amiltonian neural networks.
\newblock In \emph{NeurIPS}, 2019.

\bibitem[Hall(2015)]{hall2015lie}
Hall, B.
\newblock \emph{Lie groups, Lie algebras, and representations: an elementary
  introduction}, volume 222.
\newblock Springer, 2015.

\bibitem[Harris et~al.(2020)Harris, Millman, van~der Walt, Gommers, Virtanen,
  Cournapeau, Wieser, Taylor, Berg, Smith, et~al.]{harris2020array}
Harris, C.~R., Millman, K.~J., van~der Walt, S.~J., Gommers, R., Virtanen, P.,
  Cournapeau, D., Wieser, E., Taylor, J., Berg, S., Smith, N.~J., et~al.
\newblock Array programming with numpy.
\newblock \emph{Nature}, 585\penalty0 (7825):\penalty0 357--362, 2020.

\bibitem[Hirn et~al.(2017)Hirn, Mallat, and Poilvert]{hirn2017wavelet}
Hirn, M., Mallat, S., and Poilvert, N.
\newblock Wavelet scattering regression of quantum chemical energies.
\newblock \emph{Multiscale Modeling \& Simulation}, 15\penalty0 (2):\penalty0
  827--863, 2017.

\bibitem[Hoogeboom et~al.(2018)Hoogeboom, Peters, Cohen, and
  Welling]{hoogeboom2018hexaconv}
Hoogeboom, E., Peters, J.~W., Cohen, T.~S., and Welling, M.
\newblock {HexaConv}.
\newblock In \emph{ICLR}, 2018.

\bibitem[Huang et~al.(2019)Huang, Vaswani, Uszkoreit, Simon, Hawthorne,
  Shazeer, Dai, Hoffman, Dinculescu, and Eck]{huang2018music}
Huang, C.-Z.~A., Vaswani, A., Uszkoreit, J., Simon, I., Hawthorne, C., Shazeer,
  N., Dai, A.~M., Hoffman, M.~D., Dinculescu, M., and Eck, D.
\newblock Music {T}ransformer.
\newblock In \emph{ICLR}, 2019.

\bibitem[Hunter(2007)]{hunter2007matplotlib}
Hunter, J.~D.
\newblock Matplotlib: A 2d graphics environment.
\newblock \emph{Computing in science \& engineering}, 9\penalty0 (3):\penalty0
  90--95, 2007.

\bibitem[Ioffe \& Szegedy(2015)Ioffe and Szegedy]{ioffe2015batch}
Ioffe, S. and Szegedy, C.
\newblock Batch normalization: Accelerating deep network training by reducing
  internal covariate shift.
\newblock In \emph{ICML}, 2015.

\bibitem[Jones et~al.(2001)Jones, Oliphant, Peterson, et~al.]{jones2001scipy}
Jones, E., Oliphant, T., Peterson, P., et~al.
\newblock Scipy: Open source scientific tools for python.
\newblock 2001.

\bibitem[Katharopoulos et~al.(2020)Katharopoulos, Vyas, Pappas, and
  Fleuret]{katharopoulos2020transformers}
Katharopoulos, A., Vyas, A., Pappas, N., and Fleuret, F.
\newblock Transformers are rnns: Fast autoregressive transformers with linear
  attention.
\newblock In \emph{ICML}, 2020.

\bibitem[Kingma \& Ba(2015)Kingma and Ba]{kingma2014adam}
Kingma, D.~P. and Ba, J.
\newblock Adam: A method for stochastic optimization.
\newblock 2015.

\bibitem[Kitaev et~al.(2020)Kitaev, Kaiser, and Levskaya]{kitaev2020reformer}
Kitaev, N., Kaiser, {\L}., and Levskaya, A.
\newblock Reformer: The efficient transformer.
\newblock In \emph{ICLR}, 2020.

\bibitem[Klicpera et~al.(2019)Klicpera, Gro{\ss}, and
  G{\"u}nnemann]{klicpera2020directional}
Klicpera, J., Gro{\ss}, J., and G{\"u}nnemann, S.
\newblock Directional message passing for molecular graphs.
\newblock In \emph{ICLR}, 2019.

\bibitem[Klicpera et~al.(2020)Klicpera, Giri, Margraf, and
  G{\"u}nnemann]{klicpera2020fast}
Klicpera, J., Giri, S., Margraf, J.~T., and G{\"u}nnemann, S.
\newblock Fast and uncertainty-aware directional message passing for
  non-equilibrium molecules.
\newblock \emph{arXiv preprint arXiv:2011.14115}, 2020.

\bibitem[Kondor \& Trivedi(2018)Kondor and Trivedi]{kondor2018generalization}
Kondor, R. and Trivedi, S.
\newblock On the generalization of equivariance and convolution in neural
  networks to the action of compact groups.
\newblock In \emph{ICML}, 2018.

\bibitem[Kondor et~al.(2018)Kondor, Lin, and Trivedi]{kondor2018clebsch}
Kondor, R., Lin, Z., and Trivedi, S.
\newblock {C}lebsch--{G}ordan nets: a fully {F}ourier space spherical
  convolutional neural network.
\newblock In \emph{NeurIPS}, 2018.

\bibitem[Kosiorek et~al.(2019)Kosiorek, Sabour, Teh, and
  Hinton]{kosiorek2019stacked}
Kosiorek, A., Sabour, S., Teh, Y.~W., and Hinton, G.~E.
\newblock Stacked capsule autoencoders.
\newblock In \emph{NeurIPS}, 2019.

\bibitem[Krizhevsky et~al.(2012)Krizhevsky, Sutskever, and
  Hinton]{krizhevsky2012imagenet}
Krizhevsky, A., Sutskever, I., and Hinton, G.~E.
\newblock Imagenet classification with deep convolutional neural networks.
\newblock In \emph{NeurIPS}, 2012.

\bibitem[Lee et~al.(2019)Lee, Lee, Kim, Kosiorek, Choi, and
  Teh]{LeeLeeKim2019a}
Lee, J., Lee, Y., Kim, J., Kosiorek, A., Choi, S., and Teh, Y.~W.
\newblock Set transformer: A framework for attention-based
  permutation-invariant neural networks.
\newblock In \emph{ICML}, 2019.

\bibitem[Miller et~al.(2020)Miller, Geiger, Smidt, and
  No{\'e}]{miller2020relevance}
Miller, B.~K., Geiger, M., Smidt, T.~E., and No{\'e}, F.
\newblock Relevance of rotationally equivariant convolutions for predicting
  molecular properties.
\newblock \emph{arXiv preprint arXiv:2008.08461}, 2020.

\bibitem[Noether(1918)]{noether1971invariant}
Noether, E.
\newblock Invariant variation problems.
\newblock \emph{Zu Göttingen, Math-phys}, pp.\  235--257, 1918.

\bibitem[Oliphant(2006)]{oliphant2006guide}
Oliphant, T.~E.
\newblock \emph{A guide to NumPy}, volume~1.
\newblock Trelgol Publishing USA, 2006.

\bibitem[Oliphant(2007)]{oliphant2007python}
Oliphant, T.~E.
\newblock Python for scientific computing.
\newblock \emph{Computing in Science \& Engineering}, 9\penalty0 (3):\penalty0
  10--20, 2007.

\bibitem[Parisotto et~al.(2020)Parisotto, Song, Rae, Pascanu, Gulcehre,
  Jayakumar, Jaderberg, Kaufman, Clark, Noury, Botvinick, Heess, and
  Hadsell]{parisotto2019stabilizing}
Parisotto, E., Song, H.~F., Rae, J.~W., Pascanu, R., Gulcehre, C., Jayakumar,
  S.~M., Jaderberg, M., Kaufman, R.~L., Clark, A., Noury, S., Botvinick, M.~M.,
  Heess, N., and Hadsell, R.
\newblock Stabilizing {T}ransformers for reinforcement learning.
\newblock In \emph{ICML}, 2020.

\bibitem[Parmar et~al.(2019{\natexlab{a}})Parmar, Ramachandran, Vaswani, Bello,
  Levskaya, and Shlens]{ParRamVas2019}
Parmar, N., Ramachandran, P., Vaswani, A., Bello, I., Levskaya, A., and Shlens,
  J.
\newblock Stand-alone self-attention in vision models.
\newblock In \emph{NeurIPS}. 2019{\natexlab{a}}.

\bibitem[Parmar et~al.(2019{\natexlab{b}})Parmar, Ramachandran, Vaswani, Bello,
  Levskaya, and Shlens]{ramachandran2019stand}
Parmar, N., Ramachandran, P., Vaswani, A., Bello, I., Levskaya, A., and Shlens,
  J.
\newblock Stand-alone self-attention in vision models.
\newblock In \emph{NeurIPS}, 2019{\natexlab{b}}.

\bibitem[Paszke et~al.(2017)Paszke, Gross, Chintala, Chanan, Yang, DeVito, Lin,
  Desmaison, Antiga, and Lerer]{paszke2017automatic}
Paszke, A., Gross, S., Chintala, S., Chanan, G., Yang, E., DeVito, Z., Lin, Z.,
  Desmaison, A., Antiga, L., and Lerer, A.
\newblock Automatic differentiation in pytorch.
\newblock 2017.

\bibitem[Ramakrishnan et~al.(2014)Ramakrishnan, Dral, Rupp, and von
  Lilienfeld]{ramakrishnan2014quantum}
Ramakrishnan, R., Dral, P.~O., Rupp, M., and von Lilienfeld, O.~A.
\newblock Quantum chemistry structures and properties of 134 kilo molecules.
\newblock \emph{Scientific Data}, 1, 2014.

\bibitem[Romero \& Cordonnier(2021)Romero and Cordonnier]{romero2020group}
Romero, D.~W. and Cordonnier, J.-B.
\newblock Group equivariant stand-alone self-attention for vision.
\newblock In \emph{ICLR}, 2021.

\bibitem[Romero \& Hoogendoorn(2020)Romero and
  Hoogendoorn]{Romero2020CoAttentive}
Romero, D.~W. and Hoogendoorn, M.
\newblock Co-attentive equivariant neural networks: Focusing equivariance on
  transformations co-occurring in data.
\newblock In \emph{ICLR}, 2020.

\bibitem[Romero et~al.(2020)Romero, Bekkers, Tomczak, and
  Hoogendoorn]{romero2020attentive}
Romero, D.~W., Bekkers, E.~J., Tomczak, J.~M., and Hoogendoorn, M.
\newblock Attentive group equivariant convolutional networks.
\newblock In \emph{ICML}, 2020.

\bibitem[Ruddigkeit et~al.(2012)Ruddigkeit, van Deursen, Blum, and
  Reymond]{ruddigkeit_enumeration_2012}
Ruddigkeit, L., van Deursen, R., Blum, L.~C., and Reymond, J.-L.
\newblock Enumeration of 166 {Billion} {Organic} {Small} {Molecules} in the
  {Chemical} {Universe} {Database} {GDB}-17.
\newblock \emph{J. Chem. Inf. Model.}, 52\penalty0 (11):\penalty0 2864--2875,
  November 2012.
\newblock ISSN 1549-9596.
\newblock \doi{10.1021/ci300415d}.
\newblock URL \url{https://doi.org/10.1021/ci300415d}.
\newblock Publisher: American Chemical Society.

\bibitem[Sanchez-Gonzalez et~al.(2019)Sanchez-Gonzalez, Bapst, Cranmer, and
  Battaglia]{sanchez2019hamiltonian}
Sanchez-Gonzalez, A., Bapst, V., Cranmer, K., and Battaglia, P.
\newblock Hamiltonian graph networks with ode integrators.
\newblock \emph{arXiv preprint arXiv:1909.12790}, 2019.

\bibitem[Sch{\"u}tt et~al.(2017)Sch{\"u}tt, Kindermans, Felix, Chmiela,
  Tkatchenko, and M{\"u}ller]{schutt2017schnet}
Sch{\"u}tt, K., Kindermans, P.-J., Felix, H. E.~S., Chmiela, S., Tkatchenko,
  A., and M{\"u}ller, K.-R.
\newblock Schnet: A continuous-filter convolutional neural network for modeling
  quantum interactions.
\newblock In \emph{NeurIPS}, 2017.

\bibitem[Thomas et~al.(2018)Thomas, Smidt, Kearnes, Yang, Li, Kohlhoff, and
  Riley]{thomas2018tensor}
Thomas, N., Smidt, T., Kearnes, S., Yang, L., Li, L., Kohlhoff, K., and Riley,
  P.
\newblock Tensor field networks: Rotation-and translation-equivariant neural
  networks for {3D} point clouds.
\newblock \emph{arXiv preprint arXiv:1802.08219}, 2018.

\bibitem[Tsai et~al.(2019)Tsai, Bai, Yamada, Morency, and
  Salakhutdinov]{tsai2019transformer}
Tsai, Y.-H.~H., Bai, S., Yamada, M., Morency, L.-P., and Salakhutdinov, R.
\newblock Transformer dissection: {A}n unified understanding for
  {T}ransformer{'}s attention via the lens of kernel.
\newblock In \emph{EMNLP-IJCNLP}, 2019.

\bibitem[Van~Rossum \& Drake~Jr(1995)Van~Rossum and Drake~Jr]{van1995python}
Van~Rossum, G. and Drake~Jr, F.~L.
\newblock \emph{Python reference manual}.
\newblock Centrum voor Wiskunde en Informatica Amsterdam, 1995.

\bibitem[Vaswani et~al.(2017)Vaswani, Shazeer, Parmar, Uszkoreit, Jones, Gomez,
  Kaiser, and Polosukhin]{vaswani2017attention}
Vaswani, A., Shazeer, N., Parmar, N., Uszkoreit, J., Jones, L., Gomez, A.~N.,
  Kaiser, L., and Polosukhin, I.
\newblock Attention is all you need.
\newblock In \emph{NeurIPS}, 2017.

\bibitem[Walt et~al.(2011)Walt, Colbert, and Varoquaux]{walt2011numpy}
Walt, S. v.~d., Colbert, S.~C., and Varoquaux, G.
\newblock The numpy array: a structure for efficient numerical computation.
\newblock \emph{Computing in science \& engineering}, 13\penalty0 (2):\penalty0
  22--30, 2011.

\bibitem[Wang et~al.(2020)Wang, Li, Khabsa, Fang, and Ma]{wang2020linformer}
Wang, S., Li, B., Khabsa, M., Fang, H., and Ma, H.
\newblock Linformer: Self-attention with linear complexity.
\newblock \emph{arXiv preprint arXiv:2006.04768}, 2020.

\bibitem[Wang et~al.(2018)Wang, Girshick, Gupta, and He]{wang2018non}
Wang, X., Girshick, R., Gupta, A., and He, K.
\newblock Non-local neural networks.
\newblock In \emph{CVPR}, 2018.

\bibitem[Weiler \& Cesa(2019)Weiler and Cesa]{weiler2019general}
Weiler, M. and Cesa, G.
\newblock General {E}(2)-equivariant steerable {CNNs}.
\newblock In \emph{NeurIPS}, 2019.

\bibitem[Weiler et~al.(2018{\natexlab{a}})Weiler, Geiger, Welling, Boomsma, and
  Cohen]{weiler20183d}
Weiler, M., Geiger, M., Welling, M., Boomsma, W., and Cohen, T.~S.
\newblock {3D} steerable {CNNs}: Learning rotationally equivariant features in
  volumetric data.
\newblock In \emph{NeurIPS}, 2018{\natexlab{a}}.

\bibitem[Weiler et~al.(2018{\natexlab{b}})Weiler, Hamprecht, and
  Storath]{weiler2018learning}
Weiler, M., Hamprecht, F.~A., and Storath, M.
\newblock Learning steerable filters for rotation equivariant {CNNs}.
\newblock In \emph{CVPR}, 2018{\natexlab{b}}.

\bibitem[Worrall et~al.(2017)Worrall, Garbin, Turmukhambetov, and
  Brostow]{worrall2017harmonic}
Worrall, D.~E., Garbin, S.~J., Turmukhambetov, D., and Brostow, G.~J.
\newblock Harmonic networks: Deep translation and rotation equivariance.
\newblock In \emph{CVPR}, 2017.

\bibitem[Zaheer et~al.(2020)Zaheer, Guruganesh, Dubey, Ainslie, Alberti,
  Ontanon, Pham, Ravula, Wang, Yang, et~al.]{zaheer2020big}
Zaheer, M., Guruganesh, G., Dubey, K.~A., Ainslie, J., Alberti, C., Ontanon,
  S., Pham, P., Ravula, A., Wang, Q., Yang, L., et~al.
\newblock Big bird: Transformers for longer sequences.
\newblock In \emph{NeurIPS}, 2020.

\bibitem[Zhang et~al.(2019)Zhang, Goodfellow, Metaxas, and
  Odena]{zhang2018self}
Zhang, H., Goodfellow, I., Metaxas, D., and Odena, A.
\newblock Self-attention {G}enerative {A}dversarial {N}etworks.
\newblock In \emph{ICML}, 2019.

\bibitem[Zhong et~al.(2020)Zhong, Dey, and Chakraborty]{zhong2019symplectic}
Zhong, Y.~D., Dey, B., and Chakraborty, A.
\newblock Symplectic ode-net: Learning hamiltonian dynamics with control.
\newblock 2020.

\end{thebibliography}
\bibliographystyle{icml2021}

\appendix
\clearpage
{
\appendix
\onecolumn
{\Large \bf Appendix}

\section{Contributions}
\label{sec:contributions}
\begin{itemize}
\item Charline and Yee Whye conceived the project and Yee Whye initially came up with an equivariant form of self-attention.
\item Through discussions between Michael, Charline, Hyunjik and Yee Whye, this was modified to the current \verb!LieSelfAttention! layer, and Michael derived the equivariance of the \verb!LieSelfAttention! layer.
\item Michael, Sheheryar and Hyunjik simplified the proof of equivariance and further developed the methodology for the \verb!LieTransformer! in its current state, and created links between \verb!LieTransformer! and other related work.
\item Michael wrote the initial framework of the \verb!LieTransformer! codebase. Charline and Sheheryar wrote the code for the shape counting experiments, Michael wrote the code for the QM9 experiments, Sheheryar wrote the code for the Hamiltonian dynamics experiments, after helpful discussions with Emilien.
\item Charline carried out the experiments for Table 1, Michael carried out most of the experiments for Table 2 with some help from Hyunjik, Sheheryar carried out the experiments for Figure 3b and all the Hamiltonian dynamics experiments.
\item Hyunjik wrote all sections of the paper except the experiment sections: the shape counting section was written by Charline, the QM9 section by Michael and the Hamiltonian dynamics section was written by Emilien and Sheheryar.
\end{itemize}

\section{Formal definitions for Groups and Representation Theory}
\label{sec:def}

\begin{definition}
A \textbf{group} $\grp$ is a set endowed with a single operator $\op : \grp \times \grp \mapsto \grp$ such that
\begin{enumerate}
\item Associativity: $\forall g, g', g'' \in \grp,
(g \op g') \op g'' = g \op ( g' \op g'' )$
\item Identity: $\exists e \in \grp, \, \forall g \in \grp \, \,  g\op \id = \id \op g = g$
\item Invertibility: $\forall g \in \grp, \exists g^{-1} \in \grp, g \op g^{-1} = g^{-1} \op g = \id $
\end{enumerate}
\end{definition}

\begin{definition}
A \textbf{Lie group} is a finite-dimensional real smooth manifold, in which group multiplication and inversion are both smooth maps.
\end{definition}
The general linear group $GL(n, \mathbb{R})$ of invertible $n\times n$ matrices is an example of a Lie group.

\begin{definition}
Let $S$ be a set, and let $\operatorname{Sym}(S)$ denote the set of invertible functions from $S$ to itself. We say that a group $\grp$ \textbf{acts} on $S$ via an action
$\rho: G \rightarrow \operatorname{Sym}(S)$ when $\rho$ is a group \textbf{homomorphism}: $\rho(g_1 g_2)(s) = (\rho(g_1) \circ \rho(g_2))(s) ~ \forall s \in S$. 

If $S$ is a vector space $V$ and this action is, in addition, a \emph{linear} function, i.e. $\rho:G \rightarrow GL(V)$, where $GL(V)$ is the set of linear invertible functions from $V$ to itself, then we say that $\rho$ is a \textbf{representation} of $\grp$.
\end{definition}

\section{Proofs}
\label{sec:proof}

\begin{lemma}
\label{lemma:composition}
The function composition $f \circ f_K \circ ... \circ f_1$ of several equivariant functions $f_k$, $k \in \{1, 2, ..., K \}$ followed by an invariant function $f$, is an invariant function.
\end{lemma}
\begin{proof}
Consider group representations $\pi_1, \ldots, \pi_K$ that act on $f_1, \ldots f_K$ respectively, and representation $\pi_0$ that acts on the input space of $f_1$. If each $f_k$ is equivariant with respect to $\pi_k, \pi_{k-1}$ such that $f_k \circ \pi_{k-1} = \pi_k \circ f_k$, and $f$ is invariant such that $f \circ \pi_k = f$, then we have
\begin{align*}
    f \circ f_k \circ ... \circ f_1 \circ \pi_0 &= f \circ f_k \circ ... \circ \pi_1  \circ f_1 \\
    & ~~ \vdots \\
    & = f \circ \pi_k \circ f_k \circ ... \circ f_1 \\
    & = f \circ f_k \circ ... \circ f_1,
\end{align*}
hence $f \circ f_K \circ ... \circ f_1$ is invariant.
\end{proof}

\begin{lemma}
The group equivariant convolution $\Psi:\mathcal{I}_U \rightarrow \mathcal{I}_U$  defined as: $[\Psi f](g) \triangleq \int_{G} \psi(g'^{-1}g)f(g') dg'$ is equivariant with resepct to the regular representation $\pi$ of $G$ acting on $\mathcal{I}_U$ as $[\pi(u) f](g) \triangleq f(u^{-1} g)$.
\end{lemma}
\begin{proof}
\begin{align*}
    \Psi[\pi(u)f](g) 
    &= \int_{G} \psi(g'^{-1}g) [\pi(u)f](g') dg' \\
    &= \int_{uG} \psi(g'^{-1}g) f(u^{-1}g') dg' \\
    &= \int_{G} \psi(g'^{-1}u^{-1}g) f(g') dg' \\
    &= [\Psi f](u^{-1}g \\
    &= [\pi(u) [\Psi f]](g).
\end{align*}
The second equality holds by invariance of the left Haar measure.
\end{proof}

\eqlifting*
\begin{proof}
Note $\mathcal{L}[\pi(u)f_{\mathcal{X}}](g) = \mathtt{f}_i$ for $g \in s(ux_i)H$ and $[\pi(u)\mathcal{L}[f_\mathcal{X}]](g) = \mathcal{L}[f_\mathcal{X}](u^{-1}g) = \mathtt{f}_i$ for $g \in us(x_i)H$. Hence $\mathcal{L}[\pi(u)f_{\mathcal{X}}] = \pi(u)\mathcal{L}[f_\mathcal{X}]$ because the two cosets are equal: $s(ux_i)H = u s(x_i)H ~ \forall u \in G$. Note that this holds because:
\begin{itemize}
    \item If $g \in s(x_i)H = \{g \in G|gx_0 = x_i \}$, then $g$ maps $x_0$ to $x_i$, hence  $ug$ maps $x_0$ to $ux_i$. So if $g \in s(x_i)H$ then $ug \in s(ux_i)H= \{g \in G|gx_0 = ux_i \}$, the set of all $g$ that map $x_0$ to $ux_i$. In summary, $us(x_i)H \subset s(ux_i)H$
    \item Conversely, if $g \in s(ux_i)H$, then we know that $u^{-1}g$ maps $x_0$ to $x_i$, so $u^{-1}g \in s(x_i)H$, hence $g \in us(x_i)H$. In summary, $s(ux_i)H \subset us(x_i)H$
    \item We have shown $us(x_i)H \subset s(ux_i)H$ and $s(ux_i)H \subset us(x_i)H$, thus $s(ux_i)H = us(x_i)H$.
\end{itemize}
\end{proof}

\eqattention*
\begin{proof}
Let $\mathcal{I}_U = \mathcal{L}(G, \mathbb{R}^D)$ be the space of unconstrained functions $f: G \rightarrow \mathbb{R}^D$. We can define the regular representation $\pi$ of $G$ acting on $\mathcal{I}_U$ as follows:
\begin{equation}
    [\pi(u) f](g) = f(u^{-1} g)
\end{equation}
$f$ is defined on the set $G_f = \cup_{i=1}^n s(x_i)H $ (i.e. union of cosets corresponding to each $x_i$). Note $G_{\pi(u) f} = u G_f$, and $G_f$ does not depend on the choice of section $s$.

Note that for all provided choices of $k_c$ and $k_l$, we have:
\begin{align}
    k_c([\pi(u)f](g),[\pi(u)f](g')) &= k_c(f(u^{-1}g),f(u^{-1}g')) \\
    k_l(g^{-1}g') &= k_l((u^{-1}g)^{-1}(u^{-1}g'))
\end{align}
Hence for all choices of $F$, we have that
\begin{align}
    \alpha_{\pi(u) f}(g,g') 
    &= F(k_c([\pi(u)f](g),[\pi(u)f](g')), k_l(g^{-1}g')) \nonumber \\
    &= F(k_c(f(u^{-1}g),f(u^{-1}g')), k_l((u^{-1}g)^{-1} u^{-1}g')) \nonumber \\
    &= \alpha_{f}(u^{-1}g,u^{-1}g') \label{eq:alpha_action}
\end{align}
We thus prove equivariance for the below choice of \verb!LieSelfAttention! $\Phi: \mathcal{I}_U \rightarrow \mathcal{I}_U$ that uses softmax normalisation, but a similar proof holds for constant normalisation. Let $A_f(g,g') \triangleq \exp(\alpha_f(g,g'))$, hence Equation (\ref{eq:alpha_action}) also holds for $A_f$.
\begin{align}
    [\Phi f](g) &= \int_{G_f} w_f(g,g')f(g') dg' \\ &= \int_{G_f} \frac{A_f(g,g')}{\int_{G_f} A_f(g,g'') dg''} f(g') dg'
\end{align}

Hence:
\begin{align}
    w_{\pi(u) f}(g,g') 
    &=  \frac{A_{\pi(u) f}(g,g')}{\int_{G_{\pi(u) f}} A_{\pi(u) f}(g,g'') dg''} \nonumber \\
    &= \frac{A_f(u^{-1}g,u^{-1}g')}{\int_{uG_f} A_{f}(u^{-1}g,u^{-1}g'') dg''}  \nonumber \\
    &= \frac{A_f(u^{-1}g,u^{-1}g')}{\int_{G_f} A_{f}(u^{-1}g,g'') dg''} \nonumber \\
    &= w_f(u^{-1}g,u^{-1}g') \label{eq:w_action}
\end{align}

Then we can show that $\Phi$ is equivariant with respect to the representation $\pi$ as follows: 
\begin{align}
    \Phi[\pi(u)f](g) &= \int_{G_{\pi(u)f}} w_{\pi(u) f}(g,g') [\pi(u)f](g') dg' \nonumber \\
    &= \int_{u G_f} w_f(u^{-1}g,u^{-1}g') f(u^{-1}g') dg' \nonumber \\
    &= \int_{G_f} w_f(u^{-1}g,g') f(g') dg'\nonumber \\
    &= [\Phi f](u^{-1}g) \nonumber \\
    &= [\pi(u) [\Phi f]](g)
\end{align}
\end{proof}

Equivariance holds for any $\alpha_f$ that satisfies Equation (\ref{eq:alpha_action}). Multiplying $\alpha_f$ by an indicator function $\mathds{1}\{d(g,g') < \lambda\}$ where $d(g,g')$ is some function of $g^{-1}g'$, we can show that \textit{local} self-attention that restricts attention to points in a neighbourhood also satisfies equivariance. When approximating the integral with Monte Carlo samples (equivalent to replacing $G_f$ with $\hat{G}_f$) we obtain a self-attention layer that is equivariant in expectation for constant normalisation of attention weights (i.e. $\mathbb{E}[\hat{\Phi}[\pi(u)f](g)] = \Phi[\pi(u)f](g) =[\pi(u) [\Phi f]](g)$ where $\hat{\Phi}$ is the same as $\Phi$ but with $\hat{G}_f$ instead of $G_f$). However for softmax normalisation we obtain a biased estimate due to the nested MC estimate in the denominator's normalising constant.

\section{Introduction to Self-Attention} \label{sec:sa}
\textbf{Self-attention} \citep{vaswani2017attention} is a mapping from an input set of $N$ vectors $\{x_1, \ldots, x_N\}$, where $x_i \in \mathbb{R}^D$, to an output set of $N$ vectors in $\mathbb{R}^D$. Let us represent the inputs as a matrix $X \in \mathbb{R}^{N \times D}$ such that the $i$th row $X_{i:}$ is $x_i$. \textbf{Multihead self-attention} (MSA) consists of $M$ \textbf{heads} where $M$ is chosen to divide $D$. The output of each head is a set of $N$ vectors of dimension $D/M$, where each vector is obtained by taking a weighted average of the input vectors $\{x_1, \ldots, x_N\}$ with weights given by a weight matrix $W$, followed by a linear map $W^V \in \mathbb{R}^{D \times D/M}$.
Using $m$ to index the head ($m=1, \ldots, M$), the output of the $m$th head can be written as:
\begin{gather*}
    f^m(X) \triangleq WXW^{V,m} \in \mathbb{R}^{N \times D/M} \\
    \text{where} \quad W \triangleq \mathtt{softmax}(X W^{Q,m} (XW^{K,m})^\top) \in \mathbb{R}^{N \times N}
\end{gather*}
where $W^{Q,m}, W^{K,m}, W^{V,m} \in \mathbb{R}^{D \times D/M}$ are learnable parameters, and the softmax normalisation is performed on each row of the matrix $X W^{Q,m} (XW^{K,m})^\top \in  \mathbb{R}^{N \times N}$.
Finally, the outputs of all heads are concatenated into a $N \times D$ matrix and then right multiplied by $W^O \in \mathbb{R}^{D \times D}$. Hence MSA is defined by:
\begin{equation}
    MSA(X) \triangleq [f^1(X), \ldots, f^M(X)] W^O \in \mathbb{R}^{N \times D}.
\end{equation}
Note $X W^Q (XW^K)^\top$ is the Gram matrix for the dot-product kernel, and softmax normalisation is a particular choice of normalisation. Hence MSA can be generalised to other choices of kernels and normalisation that are equally valid \citep{wang2018non, tsai2019transformer}. 

\section{LieSelfAttention: Details} \label{sec:sa_choices}

We explore the following non-exhaustive list of choices for content-based attention, location-based attention, combining content and location attention and normalisation of weights:

\textbf{Content-based attention} $k_c(f(g),f(g'))$:
\begin{enumerate}
    \item Dot-product: $\frac{1}{\sqrt{d_v}}\left( W^Q f(g) \right)^\top W^K f(g') \in \mathbb{R} \\ ~~ \text{for} ~~ W^Q, W^K \in \mathbb{R}^{d_v \times d_v}$
    \item Concat: $\mathtt{Concat}[W^Q f(g) , W^K f(g')] \in \mathbb{R}^{2d_v}$
    \item Linear-Concat-linear: $W \mathtt{Concat}[W^Q f(g) , W^K f(g')] \in \mathbb{R}^{d_s}$ \\ for $W \in \mathbb{R}^{d_s \times 2d_v}$.
\end{enumerate}

\textbf{Location-based attention} $k_l(g^{-1}g')$ for Lie groups $G$:
\begin{enumerate}
    \item Plain: $\nu[\log(g^{-1}g')]$
    \item MLP: $\mathtt{MLP}(\nu[\log(g^{-1}g')])$
\end{enumerate}
where $\log: G \rightarrow \mathfrak{g}$ is the log map from $G$ to its Lie algebra $\mathfrak{g}$, and $\nu: \mathfrak{g} \rightarrow \mathbb{R}^d$ is the isomorphism that extracts the free parameters from the output of the log map \citep{finzi2020generalizing}. We can use the same log map for discrete subgroups of Lie groups (e.g.~$C_n \leq SO(2), D_n \leq O(2)$). See \autoref{sec:derivations} for an introduction to the Lie algebra and the exact form of $\nu\circ\log (g)$ for common Lie groups.

\textbf{Combining content and location attention} $\alpha_f(g,g')$:
\begin{enumerate}
    \item Additive: $k_c(f(g),f(g')) + k_l(g^{-1}g')$
    \item MLP: $\mathtt{MLP}[\mathtt{Concat}[k_c(f(g),f(g')), k_l(g^{-1}g')]]$
    \item Multiplicative: $k_c(f(g),f(g')) \cdot k_l(g^{-1}g')$
\end{enumerate}
Note that the MLP case is a strict generalisation of the additive combination, and for this option $k_c$ and $k_l$ need not be scalars.

\textbf{Normalisation of weights} $\{w_f(g,g')\}_{g' \in G_f}$:
\begin{enumerate}
    \item Softmax: $\softmax{\{\alpha_f(g,g')\}_{g' \in G_f}}$
    \item Constant: $\{\frac{1}{|G_f|}\alpha_f(g,g')\}_{g' \in G_f}$
\end{enumerate}

We also outline how to extend the single-head \verb!LieSelfAttention! described in Algorithm \ref{alg:lie-self-attention} extends to \textbf{Multihead equivariant self-attention}. Let $M$ be the number of heads, assuming it divides $d_v$, with $m$ indexing the head. Then the output of each head is:
\begin{equation} \label{eq:saoutput}
    V^m(g) = \int_{G_f} w_f(g,g') W^{V,m} f(g') dg' \in \mathbb{R}^{d_v/M}
\end{equation}
The only difference is that $W^{Q,m}, W^{K,m}, W^{V,m} \in \mathbb{R}^{d_v/M \times d_v}$.
The multihead self-attention combines the heads using $W^O \in \mathbb{R}^{d_v \times d_v}$, to output:
\begin{equation}
f_{out}(g) = W^O \begin{bmatrix} V^1(g) \\ \vdots \\ V^M(g) \end{bmatrix}
\end{equation}

\section{Lie Algebras and Log maps}
\label{sec:logmap}
In this section we briefly introduce Lie algebras and log maps, mainly summarising relevant sections of \citet{hall2015lie}. See the reference for a formal and thorough treatment of Lie groups and Lie algebras.

Given a Lie group $G$, a smooth manifold, its \textbf{Lie algebra} $\mathfrak{g}$ is a vector space defined to be the tangent space at the identity element $e \in G$ (together with a bilinear operation called the Lie bracket $[x,y]$, whose details we omit as it is not necessary for understanding log maps). We most commonly deal with \textbf{matrix Lie groups}, namely subgroups of the general linear group $GL(n; \mathbb{C})$, the group of all $n \times n$ invertible matrices with entries in $\mathbb{C}$. This includes rotation/reflection groups $SO(n)$ and $O(n)$, as well as the group of translations $T(n)$ and roto-translations $SE(n)$, that are isomorphic to subgroups of $GL(n+1; \mathbb{C})$. For example, $SE(n)$ is isomorphic to the group of matrices of the form:
\begin{equation*}
    \begin{bmatrix}
       & & & \vert \\
       & R & & x \\
       & & & \vert \\
     \text{---} & 0 & \text{---} & 1
    \end{bmatrix} \in  \mathbb{R}^{(n+1) \times (n+1)}
\end{equation*}
where $R \in SO(n)$ and $x \in \mathbb{R}^n$. For such matrix Lie Groups $G$, the Lie algebra is precisely the set of all matrices $X$ such that $\exp(tX) \in G$ for all $t \in \mathbb{R}$, where $\exp$ is the matrix exponential ($\exp(A) = I + A + A^2 / 2! + \ldots$). Hence the Lie algebra $\mathfrak{g}$ can be thought of as a set of matrices, and the matrix exponential $\exp$ can be thought of as a map from the Lie algebra $\mathfrak{g}$ to $G$. This map turns out to be surjective for all the groups mentioned below, and hence we may define the log map $\log: G \rightarrow \mathfrak{g}$ in the other direction. Since the effective dimension of Lie algebra, say $d$, is smaller than the number of entries of the $n \times n$ (or $(n+1) \times (n+1)$ in the case of $SE(n)$ and $T(n)$) matrix element of the Lie algebra, we use a map $\nu: \mathfrak{g} \rightarrow \mathbb{R}^d$ that extracts the free parameters from the Lie algebra element, to obtain a form that is suitable as an input to a neural network. See below for concrete examples.

\label{sec:derivations}
\begin{itemize}
    \item $G=T(n)$, $t \in \mathbb{R}^n$, $\nu[\log(t)]=t$
    \item $G=SO(2), R = \begin{bmatrix} \cos \theta & - \sin \theta \\ \sin \theta & \cos \theta \end{bmatrix} \in \mathbb{R}^{2 \times 2}$
    \begin{equation}
        \nu[\log(R)] = \theta = \arctan(R_{10}/R_{01})
    \end{equation}
    \item $G=SE(2), R = \begin{bmatrix} \cos \theta & - \sin \theta \\ \sin \theta & \cos \theta \end{bmatrix} \in \mathbb{R}^{2 \times 2}, t \in \mathbb{R}^2$
    \begin{equation}
        \nu[\log(tR)] = 
            \begin{bmatrix}
                t'\\
                \theta
            \end{bmatrix}
    \end{equation}
    where $t' = V^{-1} t$,  $V = \begin{bmatrix} a & - b \\ b & a \end{bmatrix}$, $a \triangleq \frac{\sin \theta}{\theta}$, $b \triangleq \frac{1- \cos \theta}{\theta}$
    \item $G=SO(3), R \in \mathbb{R}^{3 \times 3}, t \in \mathbb{R}^3$: 
    \begin{equation}
        \nu[\log(R)] 
        = \nu\left[\frac{\theta}{2 \sin \theta}(R - R^\top)\right] 
        = \frac{\theta}{2 \sin \theta} 
            \begin{bmatrix}
                R_{21} - R_{12} \\
                R_{02} - R_{20} \\
                R_{10} - R_{01}
            \end{bmatrix}
    \end{equation}
    where $\cos \theta = \frac{\Tr(R) - 1}{2}$. Note that the Taylor expansion of $\theta/\sin \theta$ should be used when $\theta$ is small.
    \item $G=SE(3), R \in \mathbb{R}^{3 \times 3}, t \in \mathbb{R}^3$:
    \begin{equation}
         \nu[\log(tR)] = 
            \begin{bmatrix}
                t'\\
                r'
            \end{bmatrix}
    \end{equation}
    where $t'= V^{-1} t, r'=\nu[\log(R)]$, $V=I + \frac{1 - \cos \theta}{\theta^2}(R-R^\top) + \frac{\theta - \sin \theta}{\theta^3}(R-R^\top)^2$.
\end{itemize}

\paragraph{Canonical lifting without Log map.} Recall that location-based attention only requires a function $k_l(g^{-1}g')$ which we are free to parameterise in any way. Since various groups $G$ can naturally be expressed in terms of real matrices (see above), $g^{-1}g' \in G$ can be expressed as a (flattened) real vector. For example, any $g \in SO(2)$ can simply be expressed as a vector $[t, \theta]^\intercal$ where $t \in \mathbb{R}^2$ and $\theta \in [0, 2\pi)$. Therefore, we can bypass the log map $\nu \circ \log$ and directly use this vector, which we found to be more numerically stable and sometimes resulted in better performance of \verb!LieTransformer!. In particular, for \verb!LieConv-SE2! and \verb!LieTransformer-SE2! on the Hamiltonian spring dynamics task, we did not use the $\log$ map and instead opted for this ``canonical'' lift. We plan to also try this for \verb!LieConv-SE3! and \verb!LieTransformer-SE3! for the QM9 task.

\section{Memory and Time Complexity Comparison with LieConv}
\label{sec:memory}
\subsection{LieConv}
\begin{itemize}
\item Inputs: $\{g,f(g)\}_{g \in G_f}$ where
    \begin{itemize}
        \item $f(g) \in \mathbb{R}^{d_v}$
        \item $G_f$ defined as in Section \ref{sec:method}.
    \end{itemize}
\item Outputs: $\{g, \frac{1}{|\mathtt{nbhd}(g)|} \sum_{g' \in \mathtt{nbhd}(g)} k_L(g^{-1}g') f(g') \}_{g \in G_f}$ where
    \begin{itemize}
        \item $\mathtt{nbhd}(g)=\{g' \in G_f: \nu[\log(g)] < r\}$. Let us assume that $|\mathtt{nbhd}(g)| \approx n ~\forall g$.
        \item $k_L(g^{-1}g') =  \mathtt{MLP}_{\theta}(\nu[\log(g^{-1}g')]) \in \mathbb{R}^{d_{out} \times d_v}$. 
    \end{itemize}
\end{itemize}

There are (at least) two ways of computing \verb!LieConv!: 1. Naive and 2. PointConv Trick.
\begin{enumerate}
    \item \textbf{Naive}
    \begin{itemize}
        \item[\textendash] \textbf{Memory}: Store $k_L(g^{-1}g') \in \mathbb{R}^{d_{out} \times d_v} ~ \forall g \in G_f, g' \in \mathtt{nbhd}(g)$. This requires $O(|G_f|n d_{out} d_v)$ memory.
        \item[\textendash] \textbf{Time}: Compute $k_L(g^{-1}g') f(g') \forall g \in G_f, g' \in \mathtt{nbhd}(g)$. This requires $O(|G_f|n d_{out} d_v)$ flops.
    \end{itemize}
    \item \textbf{PointConv Trick}
    
    One-line summary: instead of applying a shared linear map then summing across $\mathtt{nbhd}$, first sum across $\mathtt{nbhd}$ then apply the linear map.
    
    Details: $k_L(g^{-1}g') = \mathtt{MLP}_{\theta}(\nu[\log(g^{-1}g')]) = \mathtt{reshape}(HM(g^{-1}g'), [d_{out},d_v])$ where 
    \begin{itemize}
        \item[\textendash] $M(g^{-1}g') \in \mathbb{R}^{d_{mid}}$ are the final layer activations of $\mathtt{MLP}_{\theta}$.
        \item[\textendash] $H \in \mathbb{R}^{d_{out}d_v \times d_{mid}}$ is the final linear layer of $\mathtt{MLP}_{\theta}$.
    \end{itemize}
    The trick assumes $d_{mid} \ll d_{out} d_v$, and reorders the computation as:
    \begin{align*}
        & \sum_{g' \in \mathtt{nbhd}(g)} \mathtt{reshape}(HM(g^{-1}g'), [d_{out},d_v]) f(g') \\
        & =  \mathtt{reshape}(H, [d_{out}, d_v d_{mid}]) \sum_{g' \in \mathtt{nbhd}(g)} M(g^{-1}g') \otimes f(g')
    \end{align*}
    where $\otimes$ is the Kronecker product: $x \otimes y = [x_1y_1, \ldots x_1y_{d_y}, \ldots, x_{d_x}y_1, \ldots x_{d_x}y_{d_y}] \in \mathbb{R}^{d_x d_y}$. So $M(g^{-1}g') \otimes f(g') \in \mathbb{R}^{d_v d_{mid}}$. 
    \begin{itemize}
        \item[\textendash] \textbf{Memory}: Store $M(g^{-1}g') ~\forall g \in G_f, g' \in \mathtt{nbhd}(g)$, and store $H$. This requires $O(|G_f| n d_{mid} + d_{out} d_v d_{mid})$ memory.
        \item[\textendash] \textbf{Time}: Compute $\sum_{g' \in \mathtt{nbhd}(g)} M(g^{-1}g') \otimes f(g')$ via matrix multiplication: 
        $\begin{bmatrix}
            \vert &  & \vert \\
            M(g^{-1}g'_1) & \ldots & M(g^{-1}g'_n) \\
            \vert &  & \vert
        \end{bmatrix}
        \begin{bmatrix}
            \text{---} & f(g'_1)  &\text{---} \\
            & \vdots &  \\
            \text{---} & f(g'_n)  &\text{---}
        \end{bmatrix}$. This requires $O(d_v n d_{mid})$ flops.
        
        Then multiply by $H$, requiring $O(d_v d_{out} d_{mid})$ flops.
        
        This is done for each $g \in G_f$, so the total number of flops is $O(|G_f| d_v d_{mid} (n + d_{out}))$.
    \end{itemize}
\end{enumerate}

\subsection{Equivariant Self-Attention}
\begin{itemize}
\item Inputs: $\{g,f(g)\}_{g \in G_f}$ where
    \begin{itemize}
        \item $f(g) \in \mathbb{R}^{d_v}$
        \item $G_f$ defined as in Section \ref{sec:method}.
    \end{itemize}
\item Outputs: $\{g, f(g) + \sum_{g' \in \mathtt{nbhd}(g)} w_f(g, g') W^V f(g') \}_{g \in G_f}$ where
    \begin{itemize}
        \item $\mathtt{nbhd}(g)=\{g' \in G_f: \nu[\log(g)] < r\}$. Let us assume that $|\mathtt{nbhd}(g)| \approx n ~\forall g$.
        \item $\{w_f(g,g')\}_{g' \in G_f} = \softmax{\{\alpha_f(g,g')\}_{g' \in G_f}}$ 
        \item $\alpha_f(g,g') = k_f(f(g),f(g')) + k_x(g^{-1}g')$
        \item $k_f(f(g), f(g')) = \left( W^Q f(g) \right)^\top W^K f(g') \in \mathbb{R}$
        \item $k_x(g) = \mathtt{MLP}_{\phi}(\nu[\log(g)]) \in \mathbb{R}$
        \item $W^Q, W^K, W^V \in \mathbb{R}^{d_v \times d_v}$.
    \end{itemize}
\item \textbf{Memory}: Store $\alpha_f(g,g')$ and $W^V f(g') ~\forall g \in G_f, g' \in  \mathtt{nbhd}(g)$. This requires $O(|G_f|n d_v)$ memory.
\item \textbf{Time}: Compute $k_f(f(g),f(g'))$ and $w_f(g,g') ~\forall g \in G_f, g' \in  \mathtt{nbhd}(g)$. This requires $O(|G_f|n d_v^2)$ flops.
\end{itemize}

With multihead self-attention ($M$ heads), the output is: 
\begin{equation*}
  f(g) + W^O \begin{bmatrix} V^1 \\ \vdots \\ V^M \end{bmatrix}  
\end{equation*}
where $W^O \in \mathbb{R}^{d_v \times d_v}$, $V^m = \sum_{g' \in \mathtt{nbhd}(g)} w_f(g,g') W^{V,m} f(g')$ for $W^{K,m}, W^{Q,m}, W^{V,m} \in \mathbb{R}^{d_v/M \times d_v}$.
\begin{itemize}
    \item \textbf{Memory}: Store $\alpha_f^m(g,g')$ and $W^{V,m} f(g') ~\forall g \in G_f, g' \in  \mathtt{nbhd}(g), m \in \{1, \ldots, M\}$. This requires $O(M|G_f|n + M|G_f|n d_v/M) = O(|G_f|n (M + d_v))$ memory.
    \item \textbf{Time}: Compute $k_f^m(f(g),f(g'))$ and $w_f^m(g,g') ~\forall g \in G_f, g' \in  \mathtt{nbhd}(g), m \in \{1, \ldots, M\}$. This requires $O(M|G_f|n d_v d_v/M) = O(|G_f|n d_v^2)$ flops.
\end{itemize}

\section{Other Equivariant/Invariant building blocks}
\label{apd:norm}

\textbf{G-Pooling} is simply averaging over the features across the group:
Inputs: $\{f(g)\}_{g \in G_f}$
Output: $\bar{f}(g) \triangleq \frac{1}{|G_f|}\sum_{g \in G_f} f(g)$
Note that G-pooling is invariant with respect to the regular representation.

\textbf{Pointwise MLPs} are MLPs applied independently to each $f(g)$ for $g \in G_f$. It is easy to show that any such pointwise operations are equivariant with respect to the regular representation.

\textbf{LayerNorm} \citep{ba2016layer} is defined as follows:

Inputs: $\{g,f(g)\}_{g \in G_f}$ where
\begin{itemize}
    \item $f(g) \in \mathbb{R}^{d_v}$
    \item $G_f$ defined as in Section \ref{sec:method}.
\end{itemize}
Outputs: $\{g, \beta \odot \frac{f(g) - m(g)}{\sqrt{v(g) + \epsilon}} + \gamma \}_{g \in G_f}$ where
\begin{itemize}
    \item Division in fraction above is \textit{scalar} division i.e. $\sqrt{v(g) + \epsilon} \in \mathbb{R}$.
    \item $m(g) = \text{Mean}_c f_c(g') \in \mathbb{R}$.
    \item $v(g) = \text{Var}_c f_c(g') \in \mathbb{R}$.
    \item $\beta, \gamma \in \mathbb{R}^D$ are learnable parameters.
\end{itemize}

\textbf{BatchNorm} We also describe BatchNorm \citep{ioffe2015batch} that is used in \cite{finzi2020generalizing} for completeness:

Inputs: $\smash{\{g,f^b(g)\}_{g \in G_f, b \in \mathcal{B}}}$ where
\begin{itemize}
    \item $f(g) \in \mathbb{R}^{d_v}$
    \item $G_f$ defined as in Section \ref{sec:method}, $\mathcal{B}$ is the batch of examples.
\end{itemize}
Outputs: $\smash{\{g, \beta \odot \frac{f^b(g) - \mathbf{m}(g)}{\sqrt{\mathbf{v}(g) + \epsilon}} + \gamma \}_{g \in G_f, b \in \mathcal{B}}}$ where
\begin{itemize}
    \item Division in fraction above denotes \textit{pointwise} division i.e. $\smash{\sqrt{\mathbf{v}(g) + \epsilon} \in \mathbb{R}^D}$.
    \item $\mathbf{m}(g) = \text{Mean}_{g' \in \mathtt{nbhd}(g), b \in \mathcal{B}} f^b(g') \in \mathbb{R}^D$ - Mean is taken for every channel.
    \item $\mathbf{v}(g) = \text{Var}_{g' \in \mathtt{nbhd}(g), b \in \mathcal{B}} f^b(g') \in \mathbb{R}^D$ - Var is taken for every channel.
    \item $\mathtt{nbhd}(g)=\{g' \in G_f: \nu[\log(g)] < r\}$.
    \item $\beta, \gamma \in \mathbb{R}^D$ are learnable parameters.
\end{itemize}
A moving average of $\mathbf{m}(g)$ and $\mathbf{v}(g)$ are tracked during training time for use at test time.
It is easy to check that both BatchNorm and LayerNorm are equivariant wrt the action of the regular representation $\pi$ on $f$ (for BatchNorm, note $g' \in \mathtt{nbhd}(g)$ iff $u^{-1} g' \in \mathtt{nbhd}(u^{-1}g)$).

\section{Experimental details}
\subsection{Counting shapes in 2D point clouds}
\label{sec:constellation_setup}
Each training / test example consists of up to two instances of each of the following shapes: triangles, squares, pentagons and the "L" shape. The $x_i$ are the 2D coordinates of each point and $\mathtt{f}_i=1$ for all points.

We performed an architecture search on the \verb!LieTransformer! first and then set the architecture of the \verb!SetTransformer! such that the models have a similar number of parameters (1075k for the \verb!SetTransformer! and 1048k for both \verb!LieTransformer-T2! and \verb!LieTransformer-SE2!) and depth.

\textbf{Model architecture}.
The architecture used for the \verb!SetTransformer! \citep{LeeLeeKim2019a} consists of 8 layers in the encoder, 8 layers in the decoder and 4 attention heads. No inducing points were used.

The architecture used for both \verb!LieTransformer-T2! and \verb!LieTransformer-SE2! is made of 10 layers, 8 heads and feature dimension $d_v=128$.  The dimension of the kernel used is 12. One lift sample was used for each point.

\textbf{Training procedure}.
We use Adam \citep{kingma2014adam} with parameters $\beta_1=0.5$ and $\beta_2=0.9$ and a learning rate of $1e-4$. Models are trained with mini-batches of size 32 until convergence.

\subsection{QM9}
\label{sec:QM9_setup}
For the QM9 experiment setup we follow the approach of \citet{anderson2019cormorant} for parameterising the inputs and for the train/validation/test split. The $\mathtt{f}_i$ is a learnable linear embeddings of the vector $[1, c_i, c_i^2]$ for charge $c_i$, with different linear maps for each atom type.  We split the available data as follows: 100k samples for training, 10$\%$ for a test set and the rest used for validation. 

In applying our model to this task, we ignore the bonding structure of the molecule. As noted in \citep{klicpera2020directional} this should not be needed to learn the task, although it may be helpful as auxiliary information. Given most methods compared against do not use such information, we follows this for a fair comparison (an exception is the $SE(3)$-Transformer \citep{fuchs2020se} that uses the bonding information). It would be possible to utilise the bonding structure both in the neighbourhood selection step and as model features by treating only atoms that are connected via a bond to another atom as in the neighbourhood of that atom.

We performed architecture and hyperparameter optimisation on the $\epsilon_{HOMO}$ task and then trained with the resulting hyperparameters on the other 11 tasks. \verb!LieTransformer-T3! uses 13 layers of attention blocks (performance saturated at 13 layers), using 8 heads ($M$) in each layer and feature dimension $d_v=848$. The attention kernel uses the $linear-concat-linear$ feature embedding, identity embedding of the Lie algebra elements, and an MLP to combine these embeddings into the final attention coefficients. The final part of the model used had minor differences to the one in diagram \ref{fig:architecture}. Instead of a global pooling layer followed by a 3 layer MLP, a single linear layer followed by global pooling was used. A single lift sample was used (since $H=\{e\}$ for $T(3)$-invariant models), with the radius of the neighbourhood chosen such that $|\mathtt{nbhd}_{\eta}(g)|=50 ~ \forall g \in G$ and we uniformly sample 25 points from this neighbourhood. 
\verb!LieTransformer-SE3! used a similar hyperparameter setting, except using 30 layers (performance saturated at 30 layers) and 2 lift samples were used for each input point ($|\hat{H}|=2$), with the radius of the neighbourhood $\eta$ chosen such that the $|\mathtt{nbhd}_{\eta}(g)|=25 ~ \forall g \in G$ and we uniformly sample 20 points from this neighbourhood. All models were trained using Adam, with a learning rate of $3e-4$ and a batch size of 75 for 500 epochs. For the \verb!LieConv! models we used the hyperparameter setting that was used in \citet{finzi2020generalizing}.

Training these models with $T(3)$ and $SE(3)$ equivariance took approximately 3 and 6 days respectively on a single Nvidia Tesla-V100. 

\subsection{Hamiltonian dynamics}
\label{sec:Hamiltonian_setup}

\textbf{Spring dynamics simulation}. We exactly follow the setup described in Appendix C.4 of \citet{finzi2020generalizing} for generating the trajectories used in the train and test data.

\textbf{Model architecture}. In all results shown except Figures \ref{fig:hamiltonian-parameter-efficiency-5-step} and \ref{fig:hamiltonian-parameter-efficiency-100-step}, we used a \verb!LieTransformer-T2! with 5 layers, 8 heads and feature dimension $d_v = 160$. The attention kernel uses dot-product for the content component, a 3-layer MLP with hidden layer width 16 for the location component and addition to combine the content and location attention components. (Also, see end of \autoref{sec:logmap} for a relevant discussion about the use of the $\log$ map in location-attention $k_l$ for this task.) We use constant normalisation of the weights. We observed a significant drop in performance when, instead of constant normalisation, we used softmax normalisation (which caused small gradients at initialization leading to optimization difficulties). The architecture had 842k parameters. Our small models (with 139k parameters) in Figures \ref{fig:hamiltonian-parameter-efficiency-5-step} and \ref{fig:hamiltonian-parameter-efficiency-100-step} use 3 layers and feature dimension $d_v = 80$, keeping all else fixed. \verb!LieTransformer-SE(2)! and \verb!LieConv-SE(2)! in Figures \ref{fig:hamiltonian-parameter-efficiency-5-step} and \ref{fig:hamiltonian-parameter-efficiency-100-step} used 2 lift samples, which were deterministically chosen to be $\hat{H} = C_2 < SO(2)$, where $C_2$ is the group of $180^\circ$ rotations. In fact, this yields \textit{exact} equivariance to $T(2) \rtimes C_2$. Note that the true Hamiltonian $H(\mathbf{q}, \mathbf{p})$ for the spring system separates as $H(\mathbf{q}, \mathbf{p}) = K(\mathbf{p}) + V(\mathbf{q})$ where $K$ and $V$ are the kinetic and potential energies of the system respectively.  Following \citet{finzi2020generalizing}, our model parameterises the potential term $V$. In particular, $x_i$ is particle $i$'s position and $\mathtt{f}_i = (m_i, k_i)$ where $m_i$ is its mass and $k_i$ is used to define the spring constants: $k_ik_j$ is spring constant for the spring connecting particles $i$ and $j$ (see Appendix C.4 of \citet{finzi2020generalizing} for details).

\textbf{Training details}. To train the \verb!LieTransformer!,  we used Adam with a learning rate of 0.001 with cosine annealing and a batch size of 100. For a training dataset of size $n$, we trained the model for $400 \sqrt{3000/n}$ epochs (although we found model training usually converged with fewer epochs). When $n \leq 100$, we used the full dataset in each batch. For training the \verb!LieConv! baseline, we used their default architecture (with 895k parameters) and hyperparameter settings for this task, except for the number of epochs which was $400 \sqrt{3000/n}$ to match the setting used for training \verb!LieTransformer!.\footnote{This yields better results for \texttt{LieConv} compared to those reported by \citet{finzi2020generalizing}, where they use fewer total epochs.} The small \verb!LieConv! models (with 173k parameters) in Figures \ref{fig:hamiltonian-parameter-efficiency-5-step} and \ref{fig:hamiltonian-parameter-efficiency-100-step} use 3 layers and 192 channels (instead of the default 4 layers and 384 channels). Lastly, only for the data efficiency results in Figure \ref{fig:hamiltonian-data-efficiency}, we used early stopping by validation loss and generated \textit{nested} training datasets as the training size varies, keeping the test dataset fixed. 

\textbf{Loss computation.} One small difference between our setup and that of \citet{finzi2020generalizing} is in the way we compute the test loss. Since we compare models' losses not only over 5-step roll-outs but also longer 100-step roll-outs, we average the individual time step losses using a geometric mean rather than an arithmetic mean as in \citet{finzi2020generalizing}. Since the losses for later time steps are typically orders of magnitude higher than for earlier time steps (see e.g. Figure \ref{fig:hamiltonian-1e4-rollout}), a geometric mean prevents the losses for later time steps from dominating over the losses for the earlier time steps. During training, we use an arithmetic mean across time steps to compute the loss for optimization, exactly as in \citet{finzi2020generalizing}. This applies for both \verb!LieTransformer! and \verb!LieConv!.

\section{Additional experimental results}



\subsection{Hamiltonian dynamics}
\label{sec:hamiltonian_extra}
\begin{figure}[h]
    \centering
    \includegraphics[width=.45\textwidth, trim=0mm 0mm 0mm 16.3mm, clip]{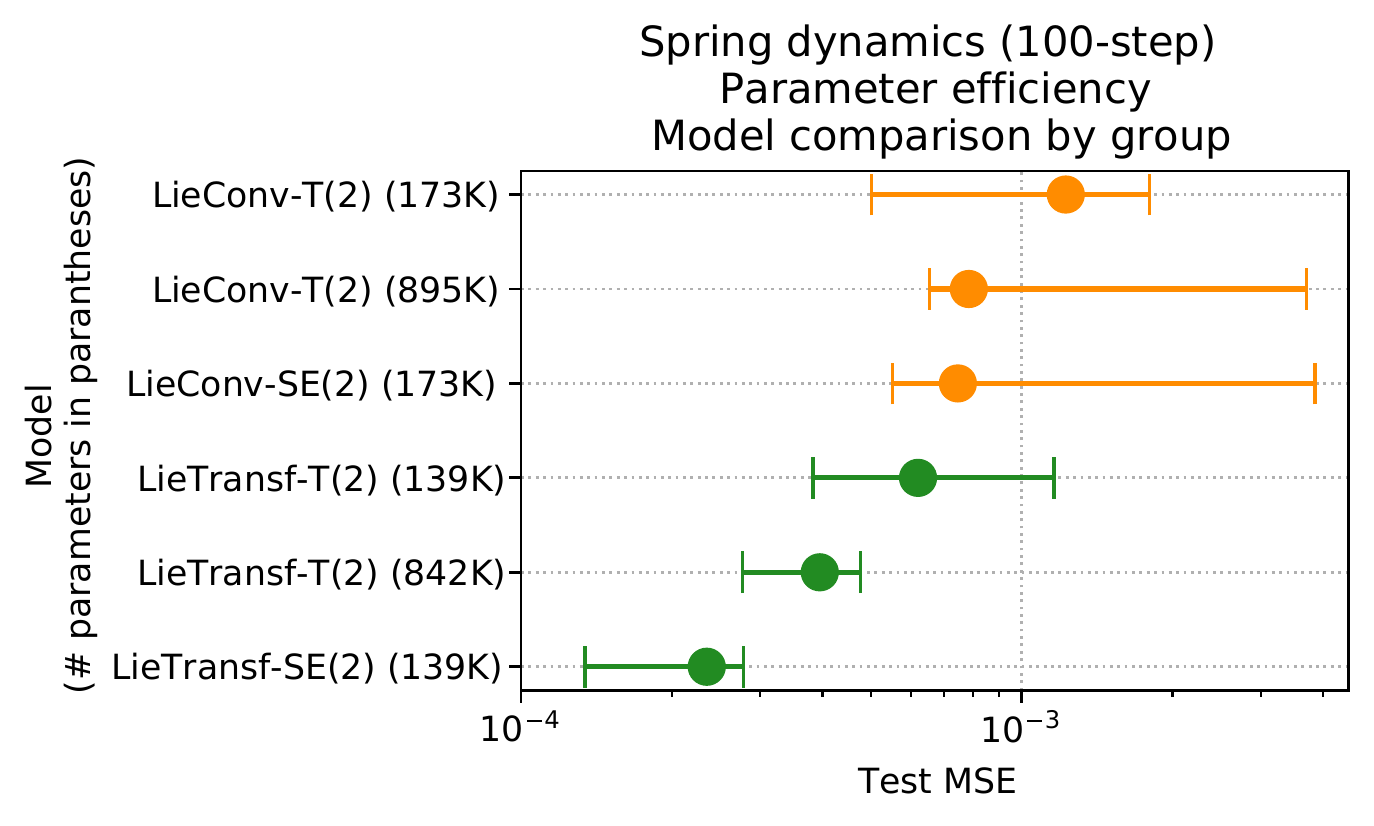}
    \caption{Comparison of models by group and number of parameters over 100-step trajectory roll-outs. Similar to the results of Figure \ref{fig:hamiltonian-parameter-efficiency-5-step}, LieTransformer models outperform their LieConv counterparts when fixing the group and using approximately equal number of parameters. Moreover, models (approximately) equivariant to SE(2) outperform their T(2) counterparts, with LieTransformer-SE(2) again outperforming all other models despite having the smallest number of parameters. Plot shows the median across at least 5 random seeds with interquartile range.}
    \label{fig:hamiltonian-parameter-efficiency-100-step}
\end{figure}

\begin{figure}[t]
    \centering
    \includegraphics[width=.99\textwidth]{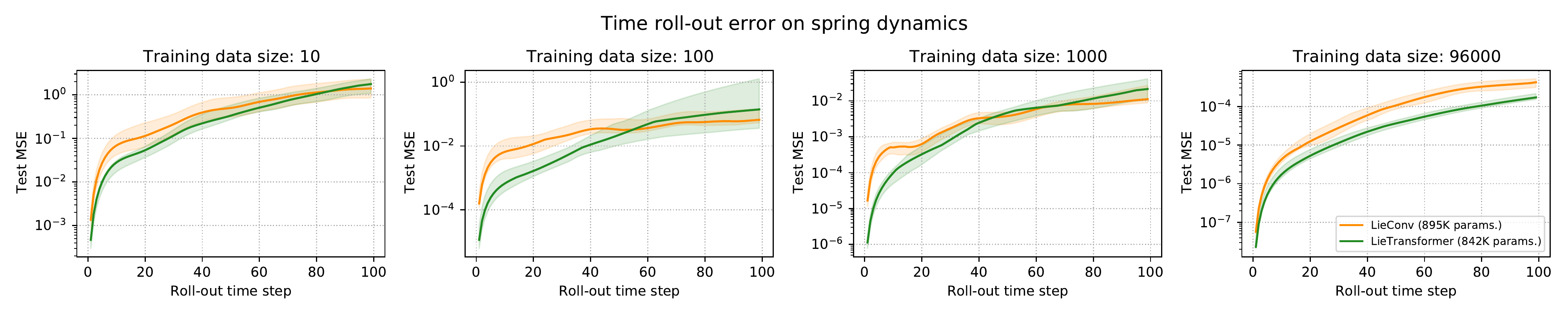}
    \caption{Plots of model error as a function of time step for various data sizes. As can be seen, the LieTransformer generally outperforms LieConv across various training data sizes.}
    \label{fig:appendix-time-roll-outs}
\end{figure}

\begin{figure}[t]
    \centering
    {\includegraphics[width=0.3\textwidth]{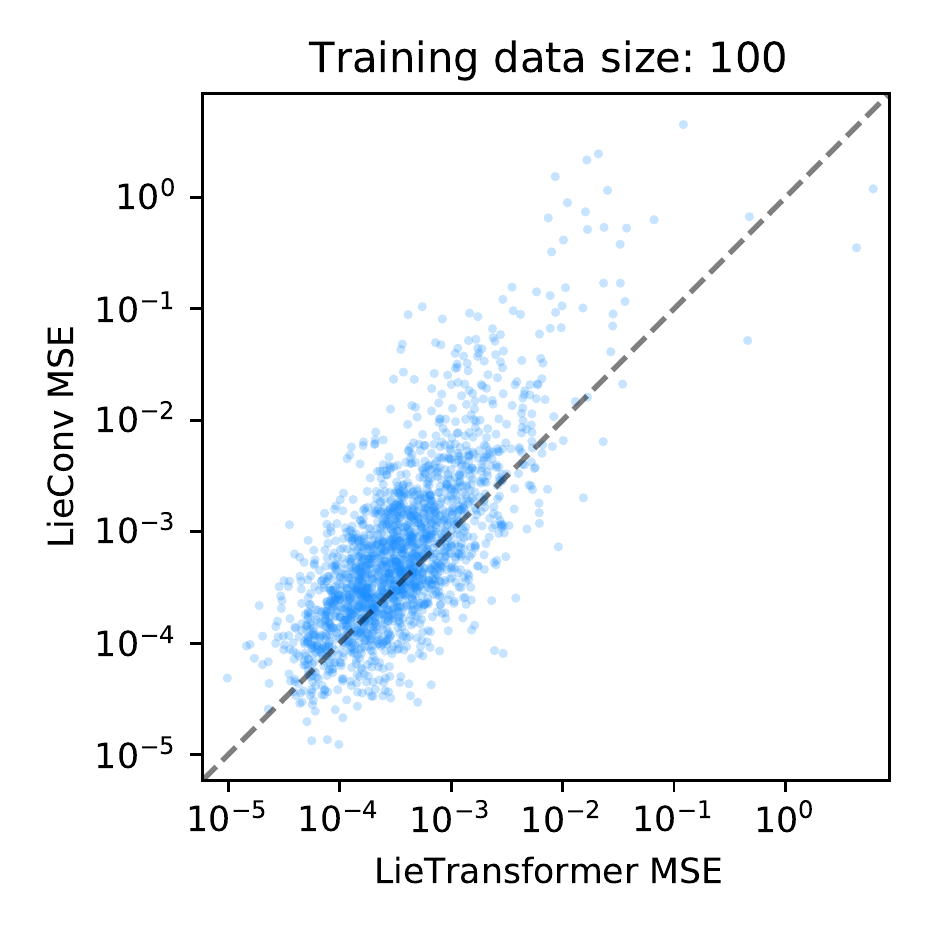}}
    \quad
    {\includegraphics[width=0.3\textwidth]{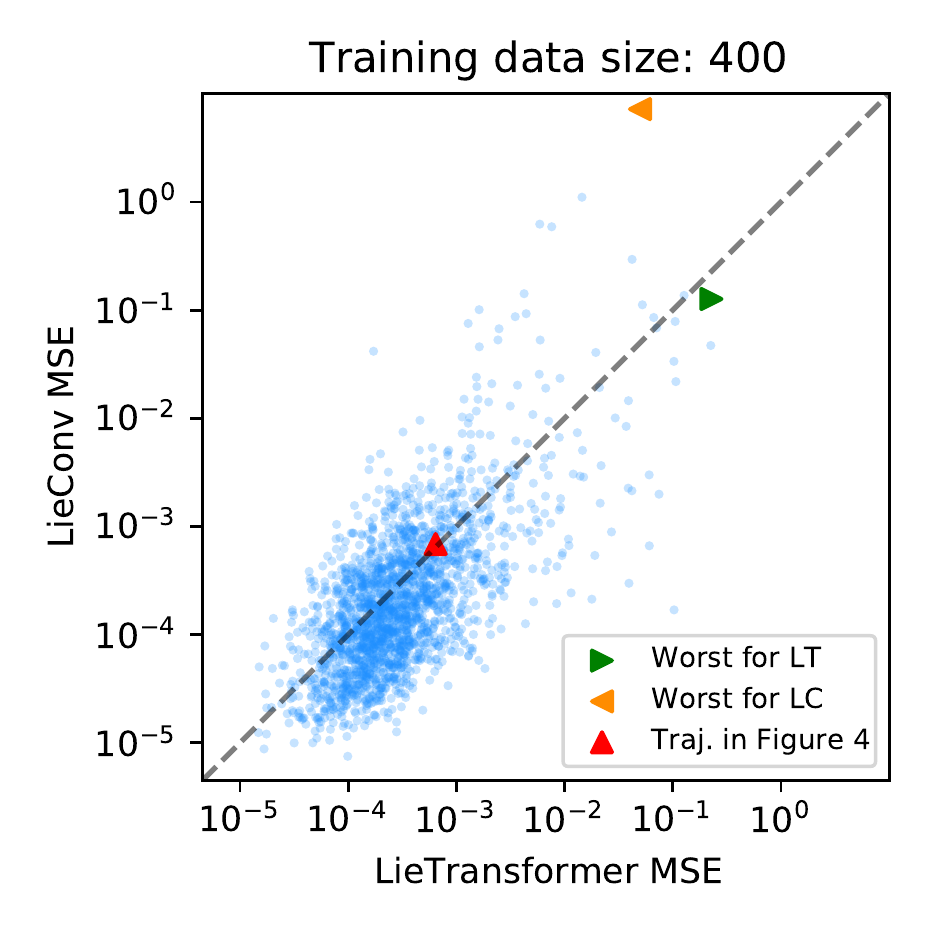}}
    \quad
    {\includegraphics[width=0.3\textwidth]{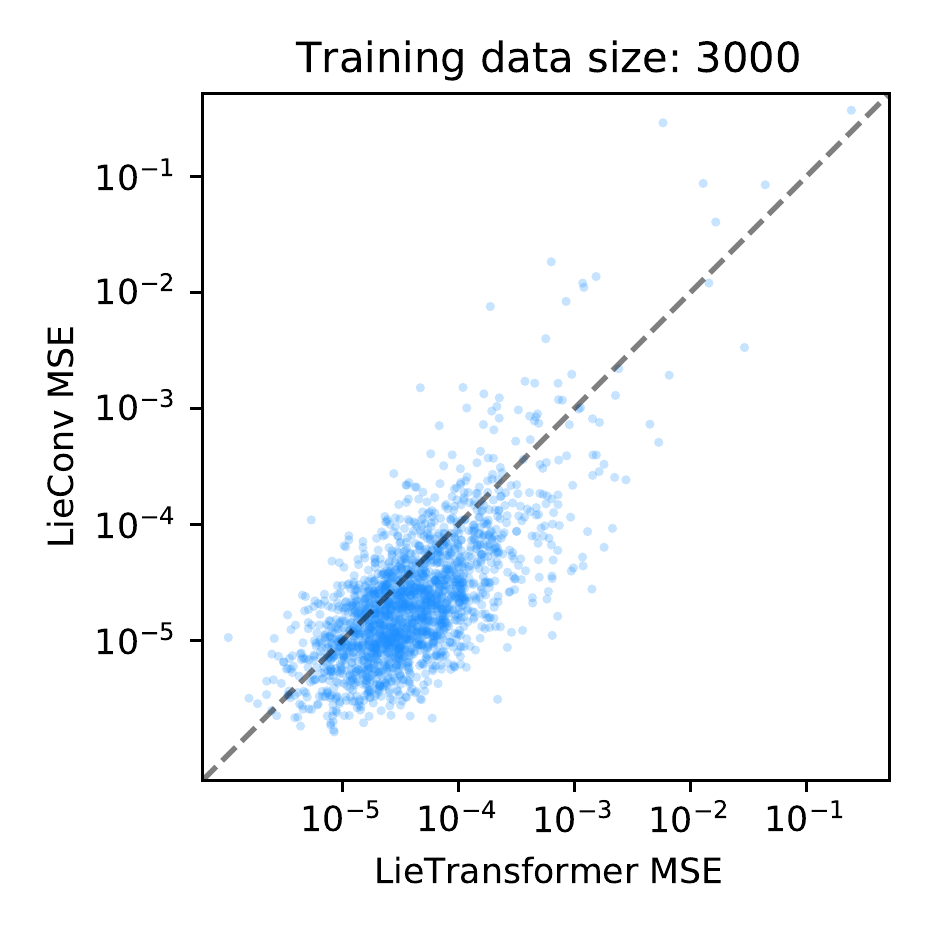}}
    \caption{Scatter plots comparing the MSE of the LieTransformer against the MSE of LieConv for various training dataset sizes. Each point in a scatter plot corresponds to a 100-step test trajectory, indicating the losses achieved by both models on that trajectory. In the middle figure we have highlighted the MSEs corresponding to the trajectories shown in Figures \ref{fig:hamiltonian-trajectories} and \ref{fig:hamiltonian-appendix-trajectories}.}
    \label{fig:traj-scatter}
\end{figure}

\begin{figure}[t]
    \centering
    \begin{subfigure}[t]{0.45\textwidth}
        \centering
        \includegraphics[width=\textwidth]{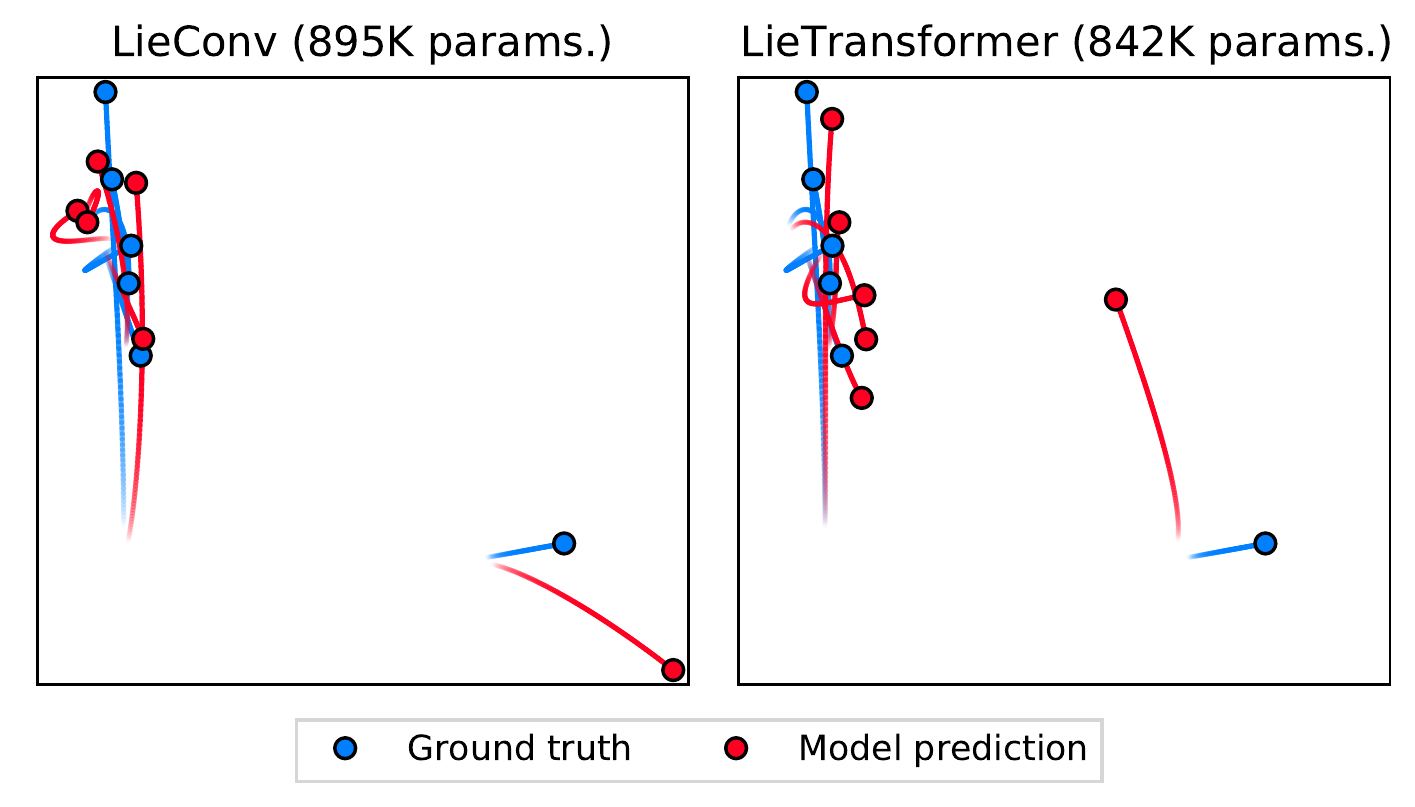}
        \caption{Test trajectory where LieTransformer has the highest error.}
    \end{subfigure}
    \quad \quad
    \begin{subfigure}[t]{0.45\textwidth}
        \centering
        \includegraphics[width=\textwidth]{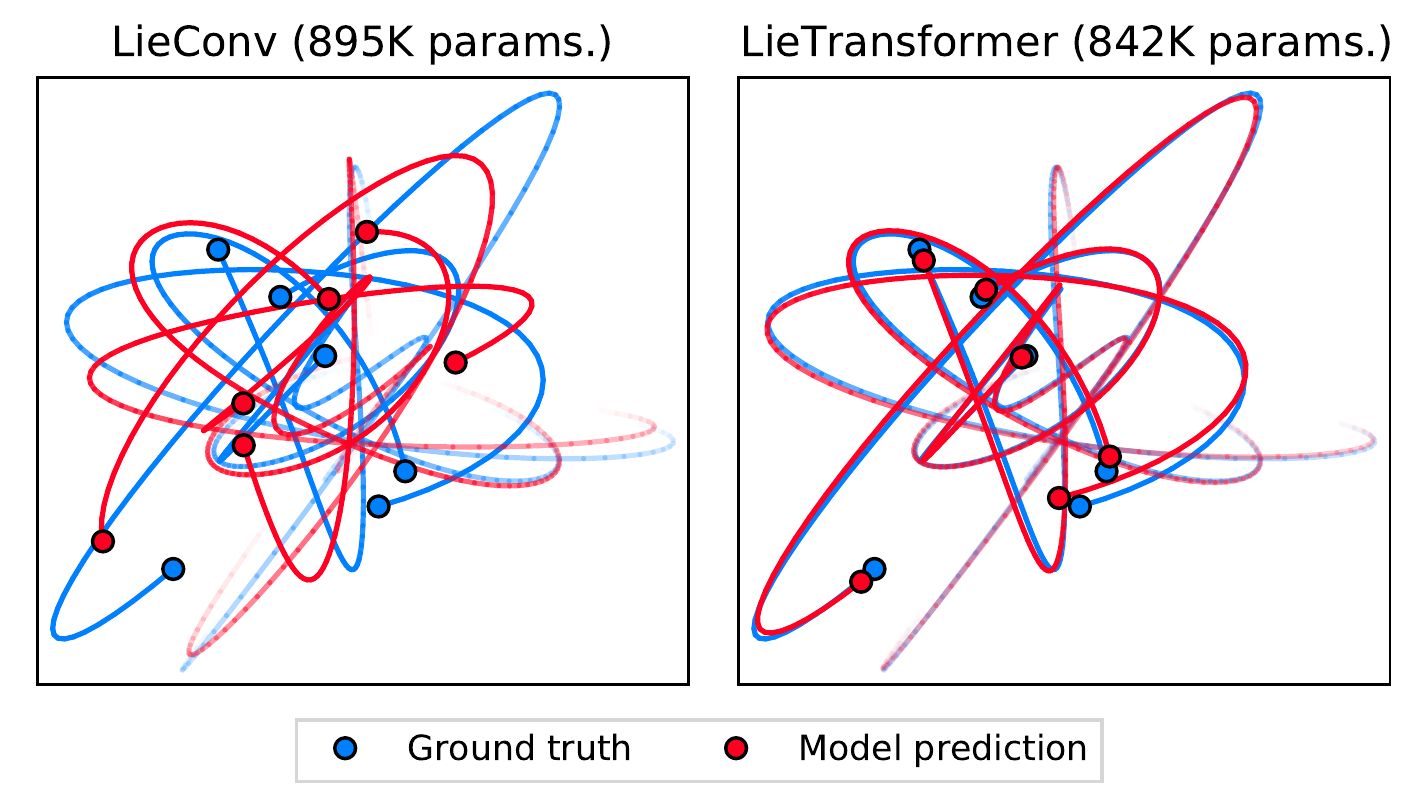}
        \caption{Test trajectory where LieConv has the highest error.}
    \end{subfigure}
    
    \caption{Additional example trajectories comparing LieTransformer and LieConv. Both models are trained on a dataset of size 400. See Figure \ref{fig:traj-scatter} for a scatter plot showing these test trajectories and the one in Figure \ref{fig:hamiltonian-trajectories} in relation to all other trajectories in the test dataset.}
    \label{fig:hamiltonian-appendix-trajectories}
\end{figure}

}  

\end{document}